\documentclass[11pt]{article}
\usepackage[letterpaper, left=1in, right=1in, top=0.95in, bottom=0.95in]{geometry}
\PassOptionsToPackage{numbers, compress}{natbib} 
\usepackage[utf8]{inputenc} % allow utf-8 input
\usepackage{multirow}
\usepackage{bbm}
\usepackage{natbib}
\usepackage{microtype}
\usepackage{graphicx}
\usepackage{subfigure}
\usepackage{booktabs} % for professional tables
\usepackage{hyperref}

\usepackage{amsmath}
\usepackage{amssymb}
\usepackage{mathtools}
\usepackage{amsthm}
\usepackage{spverbatim}
\usepackage{enumitem}

\usepackage[utf8]{inputenc} % allow utf-8 input
\usepackage[T1]{fontenc}    % use 8-bit T1 fonts
\usepackage{hyperref}       % hyperlinks
\usepackage{url}            % simple URL typesetting
\usepackage{booktabs}       % professional-quality tables
\usepackage{amsfonts}       % blackboard math symbols
\usepackage{nicefrac}       % compact symbols for 1/2, etc.
\usepackage{microtype}      % microtypography
\usepackage[dvipsnames]{xcolor}
\usepackage[suppress]{color-edits}
\usepackage{amsmath}
\usepackage{amsthm}
\usepackage{amssymb}
\usepackage{bbm}
\usepackage{graphicx}
\usepackage{spverbatim}
\usepackage{subcaption}
\usepackage[capitalize,noabbrev]{cleveref}
\usepackage{xparse}
\theoremstyle{plain}
\newtheorem{theorem}{Theorem}[section]
\newtheorem{proposition}[theorem]{Proposition}
\newtheorem{lemma}[theorem]{Lemma}
\theoremstyle{definition}
\newtheorem{definition}[theorem]{Definition}

\newcommand{\half}{{\nicefrac{1}{2}}}

\newcommand{\C}{C(\pi;y)\,}
\newcommand{\Cp}{C(1-p;y)\,}

\NewDocumentCommand{\W}{O{\half} O{y}}{W(\pi_0 \to #1;#2)}
\newcommand{\E}{\operatornamewithlimits{\mathbb{E}}}

\newcommand{\ind}[1]{{\color[rgb]{0.0,0.0,0.0}\mathbf{1}_{\left(#1\right)}}}
\NewDocumentCommand{\yhateqy}{ o o }{
  {\color[rgb]{0.0,0.0,0.0}\mathbf{1}_{\bigl(y = \kp[#1][#2]\bigr)}}%
}
\NewDocumentCommand{\kp}{ o o }{
  {\color[rgb]{0.0,0.0,0.0}\kappa(%
    \IfNoValueTF{#1}{\s(x)}{
      \IfBlankTF{#1}{\s(x)}{
      #1\otimes\shalf(x)
    }},%
    \IfNoValueTF{#2}{\tau}{
      \IfBlankTF{#2}{\tau}{
      #2
    }}%
  )}%
}
\newcommand{\clip}[2]{\underset{[#1,#2]}{\operatorname{clip}}}
\newcommand{\logit}[1]{\sigma^{\scalebox{0.5}{$-1$}\hspace{.4em}}\negthickspace\left(#1\right)}

\definecolor{gfm}{rgb}{0.8,0.4,0.0}
\definecolor{acw}{rgb}{0.0,0.4,0.8}

\newcommand{\s}{s}
\newcommand{\shalf}{s_\half}
\newcommand{\D}{\boldsymbol{\mathcal{D}}}
\newcommand{\Dsrc}{\boldsymbol{\mathcal{D}_{\pi_0}}}
\newcommand{\Dhalf}{\boldsymbol{\mathcal{D}_{\half}}}
\newcommand{\Dtgt}{\boldsymbol{\mathcal{D}_\pi}}
\newcommand{\metric}[1]{\text{#1}}

\usepackage[suppress]{color-edits}

\theoremstyle{plain}

\theoremstyle{definition}
\theoremstyle{remark}
\crefname{paragraph}{part}{parts}
\definecolor{mypink3}{rgb}{0.758, 0.088, 0.478}
\definecolor{mypink2}{RGB}{193, 48, 122}
\definecolor{mypink}{cmyk}{0, 0.7808, 0.4429, 0.1412}
\definecolor{myblue2}{rgb}{0.178, 0.188, 0.458}
\definecolor{myblue1}{HTML}{00F9DE}
\definecolor{myblue3}{rgb}{0.478, 0.488, 0.958}

\definecolor{EXAMPLE}{rgb}{0.0,0.0,0.0}
\definecolor{CAPTION}{rgb}{0.0,0.0,0.0}

\title{Aligning Evaluation with Clinical Priorities: Calibration, Label Shift, and Error Costs}

\author{
    Gerardo A. Flores\textsuperscript{1},
    Alyssa H. Smith\textsuperscript{2},
    Julia A. Fukuyama\textsuperscript{3},
    Ashia C. Wilson\textsuperscript{1}
}

\date{
\small
    \textsuperscript{1}Massachusetts Institute of Technology \\
    \textsuperscript{2}Northeastern University \\
    \textsuperscript{3}Indiana University
}

\begin{document}

\maketitle

\begin{abstract}
Machine learning-based decision support systems are increasingly deployed in clinical settings, where probabilistic scoring functions are used to inform and prioritize patient management decisions.
However, widely used scoring rules, such as accuracy and AUC-ROC, fail to adequately reflect key clinical priorities, including calibration, robustness to distributional shifts, and sensitivity to asymmetric error costs.
In this work, we propose a principled yet practical evaluation framework for selecting calibrated thresholded classifiers that explicitly accounts for the uncertainty in class prevalences and domain-specific cost asymmetries often found in clinical settings.
Building on the theory of proper scoring rules, particularly the Schervish representation, we derive an adjusted variant of cross-entropy (log score)  that averages cost-weighted performance over clinically relevant ranges of class balance.
The resulting evaluation is simple to apply, sensitive to clinical deployment conditions, and designed to prioritize models that are both calibrated and robust to real-world variations.
\end{abstract}

\section{Introduction}
The field of medicine increasingly relies on machine learning tools for clinical decision support. 
In diagnosis and prognosis, probabilistic scores capture uncertainty about patient outcomes. Combined with value judgments, they produce expected values that guide clinical decisions.
It is less often emphasized that expected value calculations can be used to measure the miscalibration of the probabilistic forecast itself.
We accordingly propose three principles that scoring functions used for clinical purposes should satisfy as closely as possible.
First, scoring functions should be adapted to account for the {\em known} {\em label shifts} that commonly arise between development and deployment environments.
In particular, many medical scoring rules are intentionally trained on more balanced class distributions than those encountered in deployment.
Second, the scores returned by scoring functions should be sensitive to the relative {\em cost of errors} that are clinically significant, such as the trade-off between the cost of misdiagnosis and the cost of failing to diagnose in any given setting. This supports patient-centered care by enabling the classifier's sensitivity to be calibrated to human feedback rather than presuming a fixed normative standard.
Third, scores should be {\em calibrated}; using them as probabilities gives practitioners easy access to decision theory as a way to consistently and reliably adapt decisions about risk and outcomes, when their clinical situation changes from the model developer's assumptions.

This work focuses on {\em evaluation}: specifically, we examine how the field of medical machine learning assesses and compares scoring functions and the extent to which current evaluation practices reflect clinical priorities.
We begin by showing that neither of the most commonly used metrics, accuracy and AUC-ROC, adequately captures all three priorities outlined above.
Each abstracts away some considerations that are critical for clinical decision-making.

We structure the paper as follows.
We first examine {\em accuracy} and its variant, balanced accuracy, as these remain the most widely used scoring rules for classification tasks.
Accuracy evaluates each decision independently and measures the overall proportion of correct predictions, abstracting away critical application-specific considerations such as {\em class imbalance} and {\em asymmetric error costs}.
While this abstraction offers a form of neutrality, it obscures important aspects of clinical deployment, where decision thresholds must often be adapted to reflect evolving prevalence rates or varying tolerances for false positives and false negatives.
As noted by several works~\cite{bradley97,provost97,costcurve06}, accuracy fixes a single operating point and, as such, fails to engage with this necessary flexibility.
In particular, it is generally not meaningful to directly compare accuracy on samples with different prevalences.

We then turn our attention to the Area Under the Receiver Operating Characteristic Curve (AUC-ROC), which is commonly viewed as a solution to the rigidity of accuracy because it evaluates classifier performance across all possible thresholds.
However, AUC-ROC measures the expected performance of the {\em ideally calibrated} version of a scoring function, not the actual, potentially {\em miscalibrated} outputs of a model.
Moreover, it ties evaluation to a distribution over positive prediction rates that may not correspond to clinical contexts.
These assumptions often lead AUC-ROC to overstate the real-world reliability of scoring functions, especially when calibration is imperfect or deployment conditions differ from development data.  The difference in units makes it hard to compare or reason together about AUC-ROC problems and miscalibration problems when both are present.

Recent work in the fairness literature has explored calibration more directly \citep{impossible17,marginalandtotalcost17,multicalibration18} but without broad consensus on best practices for how calibration interacts with varying cost structures.
In particular, the definition of perfect calibration is widely agreed upon, but the correct way to measure degrees of miscalibration, taking into account label prevalence and asymmetric error costs, is not.
As a consequence, the use of calibration-based metrics has lagged behind that of accuracy and AUC-ROC in clinical ML settings.
To address these concerns with current evaluation practices, we propose adapting a framework from the weather forecasting and belief elicitation literature known as the Schervish representation~\cite{schervish89}.
This framework shows that any proper scoring rule (a measure of calibration that doesn't require binning) can be represented as an integral over discrete cost-weighted losses, directly linking calibration to decision-theoretic performance.
We extend this framework to the setting of label shift and asymmetric costs, and average cost-sensitive metrics over a bounded range of class balances.

In summary, this work makes three main contributions.
First, we introduce a framing of scoring rule design that centers clinical priorities, namely calibration, robustness to distributional shift, and sensitivity to error costs.
Second, we use the Schervish representation to show how these priorities induce loss functions for probabilistic forecasts.
Third, we propose an adaptable scoring framework based on adjusted log scores that reflects clinical needs.
It accommodates uncertainty in class balance, asymmetric cost structures, and the requirement for calibrated predictions, thereby offering a more principled foundation for evaluating machine learning models in clinical decision support.

\subsection{Problem Formulation}
Given an input space \(\mathcal{X}\) and binary label space \(\{0,1\}\), the standard goal of binary classification is to learn a decision rule that maps each input \(x \in \mathcal{X}\) to a predicted label.
A scoring function $\s: \mathcal{X} \to \mathbb{R}$ assigns a real-valued score to each input, and a binary classifier is defined by thresholding this score.
For a threshold parameter \(\tau \in \mathbb{R}\), the predicted label is
\(
\kp = \ind{\s(x) \geq \tau}
\)
where \(\ind{\cdot}\) denotes the indicator function, equal to \(1\) if the argument is true and \(0\) otherwise.
We denote the dataset by \(\Dsrc\), consisting of input-label pairs \((x,y)\) drawn from an unknown distribution.
We define the {\em empirical class prevalence} as
\(
\pi_0 = \mathbb{P}_{\Dsrc}(y=1),
\)
which represents the proportion of positive examples in the dataset and the possibly unknown {\em target} or deployment class prevalence as \(\pi = \mathbb{P}_{\Dtgt}(y=1)\).
To formalize evaluation objectives, we introduce three additional elements: (1) A value function where \(V(y, \kp)\) specifies the utility or loss associated with predicting \(\hat{y}\) when the true label is \(y\); (2) a parameter \(c \in (0,1)\), which encodes the relative cost of false positives and false negatives and determines the threshold; and (3) a distribution \(H\) over possible data-generating distributions \(\Dtgt\), modeling uncertainty over the environment and potential distribution shifts.
We denote odds multiplication by \( a \otimes b \triangleq \frac{ab}{ab + (1-a)(1-b)}\).

\section{Related Work}
\label{apdx:relatedwork}
Recent literature emphasizes that scoring rules in AI should meaningfully reflect the objectives of deployment contexts rather than relying on standard metrics that can lead to suboptimal or misleading outcomes \cite{spiritai20, consort20}.
\vspace{-1em}
\paragraph{Decision Theory}
Decision theory has roots in gambling and actuarial sciences but was formally structured by foundational works such as \citet{ramsey26} and \citet{finetti37,finetti92}.
Within medical decision-making, a prominent recent line of inquiry has been Decision Curve Analysis (DCA), a decision theoretic framework developed by Vickers, et al. \citep{vickers06, vickers19dca}.
However, DCA has avoided measuring the area under the decision curve \citep{vickers08}, eschewing mathematical evaluation when neither classifier dominates.
This body of work critically examines widely-used metrics such as the Area Under the Curve (AUC) \citep{vickers21,vickers22} and the Brier score \citep{vickers17}, questioning their clinical utility and advocating for metrics directly connected to decision-analytic value.
\vspace{-1em}
\paragraph{Proper Scoring Rules}
The literature on proper scoring rules began with \citet{brier50}, and was subsequently enriched by contributions from \citet{good52} and \citet{mccarthy56}.
A critical advancement was the integral representation of proper scoring rules by \citet{shuford66}, which explicitly connects scoring rules with decision-theoretic utility via \citet{savage71}.
This was followed by the comprehensive characterization provided by \citet{schervish89}, who demonstrated that strictly proper scoring rules can be represented as mixtures of cost-weighted errors.
The formalism was further elucidated by \citet{shen05} through the lens of Bregman divergences and by \citet{proper07} as convex combinations of cost-sensitive metrics.
\citet{hand09} advocated adapting scoring rules explicitly to application contexts using beta distributions, a perspective later extended by \citet{hand14,recentbeta24} through asymmetric beta distributions.
Our approach differs in that we rely on uniform intervals between upper and lower bounds.
Compared to correctly setting the sum of Beta parameters, this is a more intuitive way of measuring dispersion.
\vspace{-1em}
\paragraph{Calibration Techniques}
The Pool Adjacent Violators Algorithm (PAVA), introduced by \citet{pav55}, remains a foundational calibration technique, equivalent to computing the convex hull of the ROC curve \citep{pavroc07}.
A distinct parametric calibration approach based on logistic regression was popularized by \citet{platt99}, subsequently refined for slope-only calibration by \citet{guo17calibration}.
An intercept-only version aligns closely with simple score adjustments \citep{skewcorrection02}, while broader generalizations are explored in \citet{betacalibration17}.
More recently, the calibration literature has shifted towards semi-supervised contexts, utilizing unlabeled data to enhance calibration quality \citep{bbse18,rlls19,lipton19}.
Despite extensive critiques that, for example, highlight that the widely-adopted Expected Calibration Error (ECE) \citep{ece15} is not a proper scoring rule \citep{vaicenavicius19, widmann19}, this metric remains popular in practice.
Calibration has recently emerged as a fairness metric alongside predictive accuracy.
This perspective, however, has become contentious since calibration was shown to be fundamentally incompatible with other fairness criteria \citep{impossible17}, spurring the development of "multicalibration" approaches that ensure calibration across numerous demographic subgroups \citep{multicalibration18}.
\vspace{-1em}
\paragraph{Label Shift}
Label shift techniques are a particularly useful hyponym of calibration techniques.
While the concept of shifting class prevalences without altering the underlying conditional distribution of features is longstanding \cite{meehl55,heckman74,heckman79,skewcorrection02}, formal treatments and systematic causal characterizations arose from \citet{labelshift12}.
Earlier explorations of covariate shift \citep{sugiyama08quinonero} motivated a broader field of research aimed at developing invariant representations robust to distribution shifts.
These efforts encompass methods based on richer causal assumptions \citep{saria18}, invariant representation learning \citep{bendavid10, scholkopf13}, and distributionally robust optimization \citep{dan15,dann16,sagawa20,duchi21}.
\vspace{-1em}
\paragraph{AUC-ROC \& AUC-PR} The Receiver Operating Characteristic (ROC) curve emerged within signal detection theory \citep{birdsall53,tanner53}, later becoming central in radiology and clinical diagnostics, where the convention solidified around measuring performance via the Area Under the Curve (AUC) \citep{metz78,hanley82}.
Use of AUC to aggregate over multiple thresholds was explored by \citet{spackman89}, \citet{bradley97}, and \citet{rocaccuracy05}, with subsequent critiques noting widespread interpretability issues \citep{balancedaccuracy23}.
\citet{hand09} showed how the AUC of calibrated classifiers relates to average accuracy across thresholds, while \citet{brieraucrank12} described alternative interpretations via uniform distributions of predicted score or uniform distributions of desired positive fractions (see \Cref{apdx:13} for more details).
Recently, there has been increased scrutiny of AUC-ROC, particularly regarding its lack of calibration and poor decomposability across subgroups \citep{subgroupauc19}.
Precision and Recall metrics originated in information retrieval, with Mean Average Precision (MAP) or the Area Under the Precision-Recall Curve (AUC-PR) formalized by \citet{keen1966measures,keen1968evaluation}.
While more recent trends in information retrieval have favored metrics such as Precision@K, Recall@K, and Discounted Cumulative Gain (DCG), \citet{goadrich06} popularized AUC-PR for classifier evaluation, particularly in contexts with imbalanced data.
Despite well-documented critiques --including that AUC-PR poorly estimates MAP \citep{dontintegratepr13} and lacks clear theoretical justification \citep{zhang24}-- its use persists, particularly in medical and biomedical contexts.
\vspace{-1em}
\paragraph{Cost-Sensitive Learning} Cost-sensitive evaluation, historically formalized through Cost/Loss frameworks \citep{angstrom22,murphy66}, was independently introduced in clinical decision-making as early as \citet{netcost75}.
The modern foundation of cost-sensitive learning emerged prominently in the machine learning literature in the 1980s, notably via the seminal work on MetaCost by \citet{metacost99} and the canonical overview by \citet{elkan01}.
Extending these frameworks to multi-class settings poses challenges due to the quadratic complexity of pairwise misclassification costs.
\vspace{-1em}
\paragraph{Visualization} Visualization techniques to illustrate economic or decision-theoretic value as a function of decision thresholds date back to \citet{brier55}, with subsequent development by Murphy, et al. \citep{murphy77,murphy87}, who linked visualizations explicitly to scoring rule theory.
Later rediscoveries within machine learning were articulated by \citet{hand99}, and independently by \citet{drummond00,costcurve06}.
More recently, these visualizations were generalized to include uncalibrated models \citep{briercurve11} and formally named Murphy Diagrams by \citet{ehm16}, with further implementation guidance provided by \citet{tryptich24}.

\section{Accuracy: Calibration without Label Shift Uncertainty}
The most popular metric for evaluating binary classifiers is the simplest: accuracy.
\begin{definition}[Accuracy] 
Given a dataset \(\Dsrc\), a score function $\s$, and a threshold $\tau$,  the \emph{accuracy} is defined as
$$
\metric{Accuracy}(\Dsrc, \s, \tau) = \frac{1}{|\Dsrc|} \sum_{(x,y) \in \Dsrc} \yhateqy
$$
\end{definition}
Accuracy considers the binarized score, discarding the real-valued information necessary for assessing calibration or uncertainty.
It further assumes \(V(0,1) = V(1,0)\), treating false positives and false negatives as equally costly.
This is misaligned with most real-world decision problems where asymmetric stakes are the norm.
Finally, the validity of the evaluation results presumes that the operational data-generating distribution matches the evaluation distribution, thereby ignoring the possibility of distribution shift.  We describe existing extensions to address asymmetric costs, label shift, and calibration.

\paragraph{Asymmetric Costs.}
In most practical decision problems, false positives and false negatives carry asymmetric consequences.
Several extensions of accuracy have been proposed to account for this asymmetry.
Two commonly used variants are net benefit and weighted accuracy.
This use of the term net benefit originates from decision curve analysis (DCA)~\cite{vickers06} but is similar in structure to earlier formulations~\cite{murphy66}.
We use a variation of net benefit that focuses on the benefit of true negatives rather than the costs of false positives in order to be more directly comparable to accuracy.

\begin{definition}[Net Benefit]
Given a threshold parameter \(c \in (0,1)\), the \emph{net benefit} is defined as
$$
\metric{Net Benefit}(\Dsrc, \s, \tau, c) = \frac{1}{|\Dsrc|} \sum_{(x,y) \in \Dsrc} \left(\frac{c}{1-c}\right)^{1-y} \yhateqy
$$
\end{definition}

At deployment time, the cost ratio may not match the one used in training, and the threshold may need to be adjusted.
If a score function \(\s(x)\) is well-calibrated, we can reliably threshold it to optimize binary decisions under any cost asymmetry.
Specifically, the optimal threshold \(\tau\) satisfies
\[
P(Y = 1|\s(x) = \tau) = \frac{V(0,1) - V(0,0)}{V(0,1) - V(0,0) + V(1,1) - V(1,0)},
\]
See \Cref{apdx:asymmetric_cost} for details.
Net benefit has the advantage that the interpretation of true positives remains consistent with standard accuracy.
That is, true positives are rewarded uniformly regardless of the cost ratio, while false positives are penalized according to the cost ratio determined by \(c\).
Another popular metric is \emph{weighted accuracy}, which corresponds to what \citet{murphy66} called relative utility, and is normalized so that a perfect classifier achieves a score of \(1\) regardless of class balance.
We provide a definition in \Cref{apdx:set}.
While net benefit and weighted accuracy are both widely used, both inherit critical limitations from the basic accuracy framework: they binarize the score, thereby discarding information about uncertainty and calibration, and they assume a fixed data-generating distribution, thereby failing to account for distribution shift.

\paragraph{Label Shift.}
To model deployment scenarios, we adopt a causal perspective.
Under the label shift structure $\Dtgt \to Y \to X$, the conditional distribution $P(X \mid Y, \Dtgt) = P(X \mid Y)$ remains invariant across domains.
This assumption holds in many clinical contexts, where observed features ($X$) reflect underlying conditions ($Y$) whose prevalence varies across populations ($\Dtgt$).
We focus on this structure because it aligns with the intuition of identifying latent diagnostic classes and enables robust correction methods for distribution shift.
In contrast, under the alternative structure $\Dtgt \to X \to Y$, $Y$ often encodes time-to-event outcomes, requiring distinct modeling strategies such as survival analysis.

The $\Dtgt \to Y \to X$ structure permits importance sampling to estimate deployment-time expectations.
However, because prediction is performed via $P(Y \mid X)$, we must also adjust the posterior using Bayes’ rule to account for class prevalence changes \cite{meehl55,heckman79,skewcorrection02}. 
This yields the adjusted posterior: 
\[
P(Y = 1 | \s(x), \Dtgt) = P(Y = 1 | \s(x), \Dsrc) \otimes (1-\pi_0) \otimes \pi.
\]
 This formula for the accuracy is attained in the deployment environment using the correct score adjustment.
We define $\shalf(x) \triangleq 1-\pi_0 \otimes \s(x)$, and denote the adjusted binary classifier by $\kp[\pi]$. We refer the reader to \Cref{apdx:prior_adjustment} for full derivation.

\begin{definition}[Prior-Adjusted Maximum Accuracy] Given the empirical class prevalence \(\pi_0\) and the deployment class prevalence \(\pi\), the prior-adjusted maximum accuracy is given by,
$$
\metric{PAMA}(\Dtgt, s, \tau) = \frac{1}{|\Dsrc|} \sum_{x,y \in \Dsrc}\left( \frac{\pi}{\pi_0}\right)^y\left(\frac{1-\pi}{1-\pi_0}\right)^{1-y}\yhateqy[\pi].
$$
\end{definition}

Prior-adjusted maximum accuracy allows us to handle any single known label shift, but it requires that our original score be probabilistically meaningful (what \cite{finetti37} called \emph{coherent}).
We can further combine these adjustments with asymmetric cost modeling to design metrics appropriate for specific deployment scenarios:
\begin{definition}[Prior-Adjusted Maximum Net Benefit] Given the empirical class prevalence \(\pi_0\), the deployment class prevalence \(\pi\), and the cost ratio \(c\).
$$
\metric{PAMNB}(\Dtgt, s, \tau) = \frac{1}{|\Dsrc|} \sum_{x,y \in \Dsrc}\left( \frac{\pi}{\pi_0}\right)^y\left(\frac{c}{1-c}\frac{1-\pi}{1-\pi_0}\right)^{1-y}\yhateqy[\pi][c].
$$
\end{definition}

See Appendix~\ref{apdx:set} for details and an extension to weighted accuracy.  However, even these extensions still fundamentally rely on binarized scores; we can account for any given fixed label shift, but we cannot account for label shift uncertainty.

\paragraph{Calibration}
Beyond adapting to shifts in population prevalence, another crucial aspect of evaluating probabilistic model outputs, particularly in clinical decision-making, is \emph{calibration}.
Unlike threshold-based decision-making, which focuses on classification accuracy, the goal of calibration is to ensure that predicted probabilities match observed frequencies: that is, $P(Y=1 | \s(x)) = \s(x)$. This perspective, well-established in the weather forecasting literature \citep{brier50,murphy73}, prioritizes reporting reliable probabilities over optimizing decisions directly.
However, calibration alone does not guarantee utility.
For instance, the ``climatological forecast'' which assigns the same score to all inputs is perfectly calibrated but useless for guiding decisions.
A key issue in defining calibration is specifying {\em where} it is required.
As shown in the presence of asymmetric costs, optimal decision-making depends on correctly identifying the point where $P(Y=1 | \s(x)) = c$, but this can be achieved by a model that is calibrated only at $c$.
As long as $P(Y=1 | \s(x) > c) > c$ and $P(Y=1 | \s(x) < c) < c$, a classifier can still support optimal thresholding at $c$, even if it is miscalibrated elsewhere.
Uncertain label shift is more complex and motivates a need for a broader sense of calibration.
If we can bound the possible class balances, we will need the model to be calibrated within the whole corresponding range.
A score function that is well-calibrated only in this narrow region can, however, still support robust, cost-sensitive classification.
This suggests a more nuanced perspective: rather than enforcing global calibration, it may suffice to ensure calibration within a threshold band.
Part of the contribution of this paper is to formalize and operationalize this idea.

Perfect calibration supports optimal decisions across a range of environments and objectives, but the extent to which deviations from calibration degrade performance, particularly under shift, is far less intuitive.
Developing principled ways to measure this degradation, and to evaluate classifiers in terms of local calibration and its decision-theoretic consequences, is a central motivation for the analysis that follows.

\subsection{Schervish representation}
\citet{schervish89} showed that every proper scoring rule can be represented as a mixture over cost-weighted errors, assuming that thresholds are set optimally for the associated costs.
This representation provides one of the earliest meaningful interpretations of the units of miscalibration.
Independently, \citet{hand09} rediscovered proper scoring rules, reframed as H-measures, in the context of mixtures over cost-weighted errors.
He used this framing to show that the AUC-ROC of a calibrated classifier corresponds to a mixture of cost-weighted errors under a particular (and undesirable) distribution over cost ratios.
The idea of generalizing from cost to a cost proportion that also depends on class balance has been repeatedly independently proposed in the setting where the scores' analytic distributions are known \cite{costcurve06,brieraucrank12}.
\citet{hand23} introduced the idea of a double integral over cost and balance, but their work does not explore the semantics of the resulting joint distribution, nor does it provide guidance on how the double integral should be computed.

We build on the view that proper scoring rules can be interpreted as mixtures over a distribution \( H \) of data distributions \( \Dtgt \), where each scoring rule evaluates cost-weighted errors \( V \) over the corresponding \( \Dtgt \).
Our approach does not have the ambiguity of combined cost / balance terms in \citet{costcurve06}, nor does it require the double integration over both cost and prevalence as suggested in \citet{hand23}, which produces dilogarithm terms not widely used in practice.
Instead, we fix the cost ratio \( c \) and integrate over the variability of data distributions captured by \( H \), yielding tools that are computationally simpler and semantically interpretable.

\section{AUC-ROC: Label Shift Uncertainty without Calibration}
\label{sec:aucroc}
The most common approach to integrate over a range of operating conditions is to use the AUC-ROC in place of accuracy.
This is an ordinal metric that discards information about the magnitudes of model scores and evaluates performance solely based on the relative ordering between positive and negative examples.
\begin{definition}[AUC-ROC]
Let $\s:\mathcal{X} \rightarrow \mathbb{R}$ be a scoring function on $\Dsrc$.
Then, the AUC-ROC is given by: \begin{align*}
\metric{ AUC-ROC}(\Dsrc,\s) \hspace{-.2em}
&\triangleq \hspace{-1.5em}
\sum_{(x,y) \in \Dsrc} \hspace{-.5em}\frac{1}{|\Dsrc|} \frac{1-y}{1-\pi_0}
\sum_{(x',y') \in \Dsrc} \hspace{-.5em}\frac{1}{|\Dsrc|} \frac{y'}{\pi_0}
\Big[\ind{\s(x') > \s(x)} + \tfrac{1}{2}\,\ind{\s(x') = \s(x)}\Big]
\end{align*}
\end{definition}

At first glance, this formulation poses a challenge for decision-theoretic interpretation. Specifically, it is not a priori clear how to interpret AUC-ROC within a framework where the metric corresponds to expected utility or decision quality under a specified loss function and distributional assumption.
AUC-ROC resists this interpretation because it is invariant to monotonic transformations of the score function and, therefore, indifferent to the calibration or absolute values of the scores, which are central to threshold-based decision-making.
On the other hand, AUC-ROC does capture something that accuracy fails to: it aggregates performance across the full range of the score distribution, effectively summing a population-level statistic over levels of the score.

There are numerous ways to interpret AUC-ROC, at least a dozen of which are enumerated in \Cref{apdx:13}.
We nevertheless offer a new formulation, whose proof is in \Cref{thm:auc-roc},
that sheds particular light on its relationship to label shift, i.e.
when the marginal distribution over labels differs between training and deployment.

\begin{theorem}[AUC-ROC as Accuracy Averaged Across Label Shift]
Let $\s$ be a scoring function that is calibrated on the evaluation distribution $\Dsrc$.
Then:
\begin{align*}
\metric{ \em AUC-ROC}(\s)
&=\frac{1}{2}\mathbb{E}_{t \sim \s[\Dhalf]} [\metric{\em PAMA}(\D_{1-t}, \s, \half)]
\end{align*}
where $\Dhalf$ denotes a balanced reweighting of the dataset (i.e., class prior $\pi = 1/2$), and $\s[\Dhalf]$ denotes the distribution of model scores over this reweighted set.
\end{theorem}
This perspective reveals that AUC-ROC can be viewed as averaging thresholded accuracy across a distribution of class prevalences, albeit one that is induced implicitly by the score distribution of the model itself.
This provides a limited form of robustness to label shift in contrast to metrics like accuracy which are typically evaluated at a fixed class balance.

However, this interpretation also surfaces several critical limitations.
First, AUC-ROC entirely disregards calibration.
By evaluating only the ordering of scores, it fails to assess whether predicted probabilities are well-aligned with empirical outcomes, but correctly estimating probabilities is one of the crucial pieces of expected value-based decision theory, so the lack of good estimates undermines efforts in high stakes domains.
This issue is shared by other ranking metrics such as AUC-PR and \emph{net discounted cumulative gain}, which similarly ignore score magnitudes.

The historical development of AUC-ROC provides important context.
Ordinal metrics were popularized in fields like psychology, where class prevalences were fixed by design, and information retrieval, where results per page were fixed by the capacity constraint of the querying user regardless of quality.
Their subsequent adoption in machine learning reflects a shift in evaluation priorities away from deployment evaluation and toward the abstract comparison of new architectures and optimization techniques.
In such a setting, ordinal metrics offer a convenient, threshold-free mode of comparison.
However, such metrics are poorly aligned with the needs of real-world deployments, where thresholding, cost asymmetries, and calibration are often indispensable.

Second, although AUC-ROC evaluates calibrated scores for their performance across varying class balances, the distribution over these prevalences is not user-specified or interpretable.
It is instead a byproduct of the model's score distribution on a hypothetical balanced dataset.
Consequently, the underlying population over which AUC-ROC aggregates accuracy differs across models, making metric comparisons across models trained on the same data unreliable.
Finally, AUC-ROC does not allow the independent specification of label shift and asymmetric error costs.
Although we can interpret varying prevalences as including varying cost ratios through the relationship $\pi' = 1-c \otimes \pi$ \cite{hernandez11threshold}, doing so entangles cost asymmetry with shifts in class balance.

In summary, AUC-ROC offers a partial advantage over accuracy by aggregating across class balances, but its benefits are offset by its insensitivity to calibration, its implicit and model-dependent averaging distribution, and its inability to account for cost asymmetry.
While it captures ranking performance, it fails to reflect key aspects of real-world decision quality.

\section{Log Score: Label Shift, Calibration and Cost Asymmetry}
To evaluate the utility of a thresholded classifier under uncertain or varying class balance, it is critical to evaluate the calibration of its underlying score function across a range of label distributions.
Calibration metrics are only meaningful insofar as they are expressed in units that reflect application-specific costs.
In this context, cost asymmetry is not a minor adjustment but a first-order concern that must be explicitly accounted for.
However, a persistent challenge in real-world deployments is the difficulty of comparing the impact of miscalibration, measured in cost-aligned units, with the loss in performance attributable to poor sharpness or uncertainty in ranking \cite{vickers21}.

\subsection{Background}
Accuracy can be generalized to account for asymmetric costs and label shift, but it does not provide insight into performance across varying class balances.
AUC-ROC and AUC-PR focus on performance across class balances but disregard calibration, potentially missing significant issues and offering no direct link to the ground truth.
The log score (or cross entropy) can, owing to the Schervish representation, be viewed as an average of accuracy over a range of class balances whose log odds are uniform.
Unfortunately, as~\cite{vickers17} point out, this range is vast; too broad to be clinically useful.

Where there are only a handful of deployment settings, we can take a discrete average of performance in each, but as uncertainty grows we need a simpler and more flexible, continuous approach.  Recent attempts have focused on obtaining a central estimate and fitting a Beta distribution around it \citet{hand09,hand14,recentbeta24}.  Unfortunately, the dispersion of a Beta distribution remains unintuitive to most medical (and perhaps even most ML) practitioners.
Indeed, \citet{recentbeta24} do not provide a procedure to set the pseudocount $\alpha + \beta - 2$, and~\citet{hand14} do not quantify uncertainty but suggest always using 1 as "a sensible default value".

\subsection{Clipped Cross Entropy}
The core contribution of this paper is to propose: (1) a simple way to characterize uncertainty in label shift using lower and upper bounds on class balance, (2) a straightforward means to average accuracy over that range, and (3) a natural extension to two standard approaches for handling asymmetric costs.
As previously mentioned, there exists a duality between measures of calibration and mixtures of accuracy measures across different prevalences.
This result is somewhat unintuitive;  see \Cref{apdx:cost} for a derivation that sheds more light.
We extend this result to average over only a specific subinterval of prevalences.
Our new contribution demonstrates how this can also be applied when costs are asymmetric.
These formulas enable straightforward calculation of the average cost-sensitive performance of a classifier over a specified range of prevalences.

We begin by specifying a lower bound $a$ and an upper bound $b$ on the class balance.
These bounds can be obtained in direct consultation with domain experts, through surveys, or utilizing previous estimates.
In many cases, even very conservative bounds will substantially reduce the range of possible prevalences.
However, since there can be order-of-magnitude differences in the prevalence of a condition across different populations, we average uniformly over the log odds of the prevalence rather than linearly on the class prevalence itself.
This means, for example, that the interval between one part in one hundred and one part in ten thousand will be about half above and half below one part in one thousand.
We represent this log odds transformation by saying that $\logit{\pi} \sim \text{Uniform}(\logit{a}, \logit{b})$
where $\sigma(x) \triangleq \tfrac{1}{1+e^{-x}}$ is the typical sigmoid function.
In what follows, we define the normalization constant $\gamma$ based on the cost ratio $c$ and the bounds on class balance $a$ and $b$.

\begin{theorem}[Bounded DCA Log Score]
\label{thm:dca-log}
Let \( \s(x) \in [0,1] \) be a score function and \( c \in (0,1) \) be a cost parameter defining a decision threshold.
Then, the expected net benefit over a logit-uniform prior on prevalence satisfies
\[
\E [\metric{\em PAMNB}(\Dtgt, \s, \tau, c)]
\hspace{-.3em}= \hspace{-.3em}\gamma (
\E [\log \hspace{-.2em}|
  1\hspace{-.2em}-\hspace{-.15em}y\hspace{-.15em}-\hspace{-1em}\clip{1-b}{1-a}\hspace{-1em}(1\hspace{-.2em}-\hspace{-.1em}c\otimes\shalf(x))
|
\hspace{-.2em}-\hspace{-.2em}\log\hspace{-.2em} |
  1-\hspace{-.15em}y\hspace{-.15em}-\hspace{-1em}\clip{1-b}{1-a}\hspace{-1em}(1\hspace{-.15em}-\hspace{-.15em}y)
|
])
\]
where the expectation on the left-hand side is with respect to $\logit{\pi} \sim \text{Uniform}(\logit{a}, \logit{b})$,  and on the right-hand side is with respect to $(x,y) \sim \D_{(1-c)}$,  and where  $\gamma \triangleq \tfrac{2 (1-c)^{-1}}{\logit{b} - \logit{a}}$.
\end{theorem}

By clipping the score, this formula is able to produce a closed form expression for the average net benefit over a range of class balances.
A major practical benefit of this score is that it is based on a pointwise calculation of loss, so confidence intervals can be trivially bootstrapped by resampling calculated losses, and each new draw only requires a weighted sum of these losses.
However, as previously discussed, the ability to handle asymmetric costs is a first-order consideration when using such units.
We focus on net benefit, because it is scaled the same way at different class balances, so it makes sense to add and average this quantity.
However, see \Cref{thm:wa-log} for a related derivation that holds if we want to use weighted accuracy, which is effectively a rescaled version of net benefit for which the maximum possible score of a perfect classifier is always 1, regardless of the class balance.

\subsection{Calibration, Label Shift and Asymmetric costs.}

The units of net benefit for any given class balance are clear: they are denominated by the value of a true positive.  Our newly introduced DCA log score is a mixture over class balances, but remains in units of true positives.
The Schervish representation gives us a simple way to describe what this is doing as a calibration metric as well; it is calibrating the model only over a particular bounded range of scores that is relevant to decisions and weighting scores in that range uniformly in log odds space.
This, then, is a measure of miscalibration that can be used to directly compare classifiers' effects in the world.
Indeed, \citet{proper07} showed that the unconstrained log score can be decomposed linearly into components of calibration and sharpness.
As \cite{vickers21} has argued, the need to weigh failures of calibration against failures of sharpness is a perpetual problem in the deployment of machine learning models in a medical context.

Moreover, our approach is flexible enough that we can generalize beyond accuracy (a well known result) to cost-sensitive metrics, as we demonstrate with DCA Log Score and Weighted Accuracy Log Score. See \Cref{apdx:log} for details.

\section{Application to Subgroup Decomposition}
The utility of our clipped cross entropy approach is highlighted in analyzing the following racial disparities in Accuracy and AUC-ROC on the publicly available subset of the eICU dataset \cite{eicu21}, for predictions of in-hospital mortality using APACHE IV scores.

\begin{figure}[ht]
\begin{minipage}{0.3\textwidth}
\includegraphics[width=\textwidth]{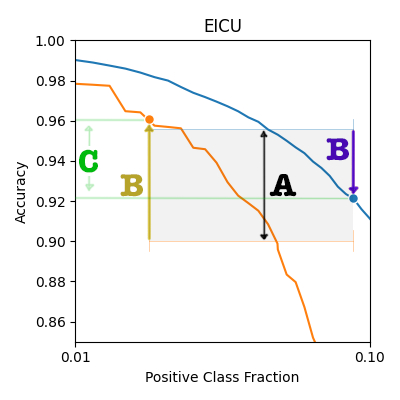}
\end{minipage}
\begin{minipage}{0.69\textwidth}
  (C) African American patients (orange) have noticeably better AUC-ROC than Caucasian patients (blue).
  However, we can decompose the difference in accuracy into (A) a difference in mechanism of prediction at equal class balances (i.e. same in-hospital mortality) and (B) a difference in the class balance at which accuracy is evaluated for the two groups.
  Doing so reveals that across the range of prevalences, performance is consistently lower for African American patients, indicating that the observed accuracy difference is entirely driven by label shift (B).
\end{minipage}
\end{figure}
\vspace{-1em}
\begin{figure}[ht]
\begin{minipage}{0.69\textwidth}
  (C) African American patients (orange) have noticeably better AUC-ROC than Caucasian patients (blue).
  However, we can plot the accuracy of a perfectly recalibrated model (dashed lines), and then decompose the average accuracy using the calibration–sharpness framework \cite{shen05,proper07}.
  We see that (A) the model gives sharper predictions for African American than Caucasian patients, but (B), it is badly miscalibrated for African American patients and has virtually no miscalibration loss for Caucasian patients.
  The most important aspect of this analysis is that we can directly compare the magnitudes of the two effects.
\end{minipage}
\begin{minipage}{0.3\textwidth}
\includegraphics[width=\textwidth]{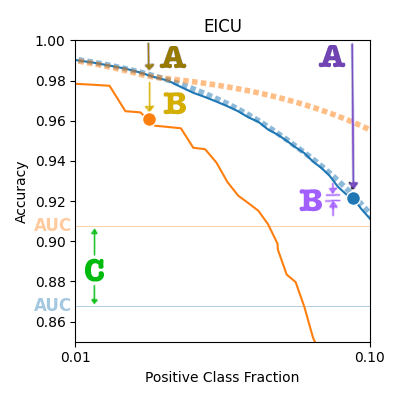}
\end{minipage}
\end{figure}

\vspace{1em}
Our analysis illustrates that ranking-based metrics fail to address miscalibration in practical scenarios in a way that cannot be readily fixed by adding a separate and incommensurable calibration metric.
Furthermore, they do not allow decompositions to consider the effects of differing class balance vs differing mechanisms of prediction.
Conversely, it shows that accuracy can be misleading because of its dependence on specific class balances, and that the average over a range of class balances gives a more intuitive picture of the reasons for gaps in performance.

In this particular example, we quantify uncertainty by computing confidence intervals (see \Cref{apdx:eicu} for visualization).
Given that mortality prevalence is around 10\%, and African American patients make up roughly 10\% of the dataset, we lack sufficient statistical power to generalize conclusions from the public subset to the full data.
However, the example is illustrative of the ways in which the analytical flexibility of the Schervish approach supports both practical development and principled scrutiny of deployed models.

\section{Discussion}

The prevailing paradigm for evaluating medical ML decision‐support systems often misaligns with evidence‐based medicine and beneficence by overlooking real‐world cost structures, disease prevalences, and calibration nuances.  We address these gaps through two main contributions:

\begin{enumerate}[leftmargin=*]
  \item \textbf{Causal distribution-shift grounding of AUC–ROC and accuracy.}  
    We show that AUC–ROC corresponds to the expected utility of a fixed threshold under a specific distribution over class prevalences, and that accuracy arises as its degenerate case at the cost ratio in the evaluation set.
  \item \textbf{Illustration of the Schervish Representation}
    Inspired by Schervish’s insight that “scoring rules are just a way of averaging all simple two‐decision problems into a single, more complicated, decision problem”~\citep{schervish89}, we reconceptualize calibration as a continuum of cost‐weighted binary decisions.
    This perspective clarifies when AUC–ROC and accuracy serve as valid proxies for clinical benefit (and when they obscure important cost asymmetries and class imbalances) and motivates their augmentation with interval‐specific calibration metrics like the DCA log score for deployment‐critical evaluations.
  \item \textbf{DCA log score for cost‐sensitive calibration.}  
    Unlike binned calibration or global metrics, the DCA log score isolates miscalibration over clinically relevant probability intervals that are dictated by anticipated cost ratios and base rate bounds, thereby making the practical impact of calibration errors explicit.
\end{enumerate}

Our framework further elucidates a conceptual tension between forecasting and classification uncertainty via their causal structures that has perhaps limited the uptake of evaluation measures used in the forecasting literature, such as Brier scores, in the medical setting.  Forecasting \(\Dtgt \rightarrow X \rightarrow Y\) assumes stability in \(P(Y\mid X)\) but is vulnerable to feature‐distribution shifts, while classification \(\Dtgt \rightarrow Y \rightarrow X\) assumes stable \(P(X\mid Y)\), enabling more robust performance predictions across thresholds and prevalences.  Recognizing this causal reversal underpins the targeted calibration assessment we propose.

\subsection{Limitations and Future Work}
While this work contributes a flexible and decision-theoretically grounded framework for evaluating predictive models under cost asymmetry and distributional shift, several challenges remain. These limitations point to open areas for theoretical refinement, methodological innovation, and practical implementation. Below, we highlight key directions for future research.

\textbf{Cost Uncertainty.} While our extension of the Decision Curve Analysis (DCA) log score to uncertain cost ratios captures realistic ambiguity in clinical tradeoffs, it introduces dilogarithmic expressions that are analytically opaque and computationally intensive. These forms limit practical interpretability and scalability. Future work could explore tractable approximations or surrogate objectives that preserve sensitivity to cost uncertainty while enabling smoother optimization and interpretive clarity.

\textbf{Sampling Variability under Label Shift.} In settings with symmetric misclassification costs, bootstrap resampling or binomial confidence intervals suffice for uncertainty estimation. However, under asymmetric costs, especially with population label shift, the evaluation metrics become sensitive to multinomial fluctuations in both the score distribution and the cost-weighted outcome prevalence. This introduces high variance and potential estimation bias. Quantifying and stabilizing this variability remains a significant challenge.

\textbf{Adaptive Base Rate Estimation.} Our framework presumes known deployment class prevalences. In practice, these rates may be uncertain or drift over time due to changing patient populations, care protocols, or screening policies.  Jointly estimating prevalence and adjusting probabilistic predictions in such regimes introduces an additional source of uncertainty.  Future work could combine the error properties of the prevalence estimation along with threshold selection and cost evaluation.

\textbf{Asymmetric Cost Parameterization.} We adopt a general framework for asymmetric cost modeling, but the semantics of varying cost ratios remain under-theorized.  At a fixed cost ratio they all produce the same results, but different parameterizations introduce different scaling factors into the overall costs, as well as changing the meaning of "uniform uncertainty", leading to different properties when averaging.  A systematic comparative study of these properties could yield robust and usable guidelines for choosing a parameterization.
\medskip

By combining decision-theoretic tools, causal framing, and clinically grounded metrics of calibration, this work moves toward evaluation methodologies that are both conceptually principled and actionable in real-world medical settings. Continued advances will require deeper integration of uncertainty quantification, model adaptivity, and domain-informed cost modeling.

\section*{Acknowledgements} This work was generously supported by the MIT Jameel Clinic in collaboration with Massachusetts General Brigham Hospital.

\bibliographystyle{abbrvnat}
\bibliography{references}

\appendix
%!TEX root = ../main.tex
\section{Calibration}
\label{apdx:calibration}
The weather forecasting literature focuses on what are known as strictly proper scoring rules: those metrics that have the property that a forecaster is correctly incentivized to report their actual beliefs about the probability of the event.  At first glance this seems a bit distant from binary classifier evaluation.  After all, action is generally binary; we really want the weather report to tell us whether to take an umbrella, not to give us 3 decimal places of precision on the long run frequency with which it would rain.

\subsection{Asymmetric Cost}
\label{apdx:asymmetric_cost}
However, a calibrated, thresholded binary classifier has an immensely useful property: we know how to change the threshold to trade off false positives for false negatives if the class balance or the cost ratio changes.  The optimality condition for choosing a threshold requires that the first derivative of the expected value be zero.  This is equivalent to saying that the expected utility of assigning points exactly at the threshold to either class should be the same:
\begin{align*}
\E_{x,y:\s(x)=\tau} V(y,0) &= \E_{x,y:\s(x)=\tau} V(y,1)
\\\implies P(y=1|\s(x)=\tau) &= \frac{V(0,1)-V(0,0)}{(V(0,1)-V(0,0))+(V(1,1)-V(1,0))}
\end{align*}
As a result, we generally call the quantity on the right $c$ and use it to describe the asymmetry of the error costs.

Solving this for $\tau$ requires us to at least implicitly estimate $\tilde{\s}(\tau) = P(y=1|\s(x)=\tau)$.  If we add the constraint of monotonicity to $\tilde{\s}(\tau)$, then this problem is known as isotonic regression, and there are well-known algorithms for solving it.  Assuming the existence of a good estimator $\tilde{s}(\tau)$ for this quantity, then $\tilde{\s}(\s(x))$ is of course a calibrated estimator for $P(y=1|x)$.  Using the same classifier at varying cost asymmetries requires that the classifer be, at minimum, implicitly calibrated; isotonic regression is in fact how such classifiers are calibrated.

It is, of course, possible to develop an estimator calibrated only at $\s(x)=c$; for any point higher or lower, ordinal comparison alone is enough to make a decision.  A classifier optimized in this fashion may be wildly unreliable at other values of $\s(x)$, and our calibration may simply give us two scores: higher than $c$ and lower than $c$.  If so, the condition of calibration is almost trivially satisfied: it only requires a binary predicted label and statistics for how often the classifier is correct in either case (PPV and NPV).

\subsection{Label Shift}
\label{apdx:label_shift}
In the weather forecasting literature, it is explicitly understood that $\Dtgt \rightarrow X \rightarrow Y$, which is to say that rather than today's atmospheric conditions being emanations of tomorrow's decision of whether to rain or not, the evolution on physical principles of today's conditions leads to tomorrow's weather.  As such, the idea of label shift is incoherent.  The study of changes in classifier performance when $P(y)$ changes in this setting is known as covariate shift; this is out of scope for this paper.

The machine learning evaluation literature does acknowledge links between label shift and calibration.
However, the setting is more abstract, with CDFs taken as given, and varying thresholds interpreted as a response to varying class balances without a clear enumeration of assumptions.

\subsubsection{Notation}
\label{apdx:label_shift_notation}

\begin{lemma}[Importance Sampling as $\ell_1$ distance]
\label{lem:sample}
Consider the standard importance sampling weights to move from the training ($\pi_0$) to the deployment ($\pi$) distribution.  Label shift always holds when reweighting data by class because after we stratify by class, we do not change the distribution within the class.
\begin{align*}
\W[\pi][y'] &\triangleq \frac{\mathbb{P}_\pi(x,y=y')}
{\mathbb{P}_{\pi_0}(x,y=y')}
\\&= \underbrace{\frac{\mathbb{P}_\pi(x|y=y')}
{\mathbb{P}_{\pi_0}(x|y=y')}}_{\text{one, by label shift}}
\frac{\mathbb{P}_\pi(y=y')}
{\mathbb{P}_{\pi_0}(y=y')}
\\[1em]&= \frac{\half}{\mathbb{P}_{\pi_0}(y=y')}
\frac{\mathbb{P}_\pi(y=y')}{\half}
\\[1em]&= \W[\half][y'] \cdot 2 \mathbb{P}_\pi(y=y')
\\[1em]&= \W[\half][y'] \begin{cases}
2(\pi - 0) & \text{if } y'=1
\\[1em]
2(1 - \pi) & \text{if } y'=0
\end{cases}
\\[1em]&= \W[\half][y']\;2\;|(1-\pi) - y'|
\end{align*}
\end{lemma}

\begin{definition}[Odds Multiplication]
\[
a \otimes b \triangleq \frac{ab}{ab + (1-a)(1-b)}
\]
\end{definition}

\begin{proposition}[Inverse]
\[
  a \otimes b = c \iff b = (1-a) \otimes c
\]
\end{proposition}

\begin{proposition}[Jacobian]
\[
  a \otimes b \frac{da}{a(1-a)} = a \otimes b \frac{d(a\otimes b)}{(a\otimes b)(1-a\otimes b)}
\]
\end{proposition}

\begin{proposition}[One minus distributes over odds multiplication]
\[
1 - (a \otimes b) = 1-a \otimes 1-b
\]
\begin{proof}
\begin{align*}
\frac{ab}{ab + (1-a)(1-b)}
+ \frac{(1-a)(1-b)}{ab + (1-a)(1-b)}
&=1
\\
a\otimes b
+
\quad
(1-a)\otimes(1-b)
&=1
\\
(1-a)\otimes(1-b)
&=
1 - a\otimes b
\end{align*}
\end{proof}
\end{proposition}

\begin{proposition}[Logit Odds Multiplication is Additive]
\[
\logit{a \otimes b} = \logit{a} + \logit{b}
\]
\begin{proof}
\[
\log\frac{
  \frac{ab}{ab + (1-a)(1-b)}
}{
  \frac{(1-a)(1-b)}{ab + (1-a)(1-b)}
}
= \log\frac{a}{1-a} + \log\frac{b}{1-b}
\]
\end{proof}
\end{proposition}

\begin{proposition}[Log Odds Interval Invariance]
\label{prop:logit-interval}
\[
\logit{1-c\otimes b} - \logit{1-c\otimes a}
=
\logit{b} - \logit{a}
\]
\begin{proof}
\begin{align*}
&\logit{1-c\otimes b} - \logit{1-c\otimes a}
\\&=
[\logit{1-c} + \logit{b}] - [\logit{1-c} + \logit{a}]
\\&=
[\logit{1-c} - \logit{1-c}] + [\logit{b} - \logit{a}]
\\&=
[\logit{b} - \logit{a}]
\end{align*}
\end{proof}
\end{proposition}

\subsection{Prior-Adjustment}
\label{apdx:prior_adjustment}

With this notation, working with conditional probabilities is straightforward:
\begin{align*}
&P(y=1|\s(x),\Dtgt)
\\&= \frac{
  \overbrace{P(\s(x)|y=1,\Dtgt)}^\text{collect this}P(y=1|\Dtgt)
}{
  \underbrace{P(\s(x)|y=1,\Dtgt)}_\text{and this}P(y=1|\Dtgt)
  +
  \underbrace{P(\s(x)|y=0,\Dtgt)}_\text{and this}P(y=0|\Dtgt)
}
\\&= \underbrace{\frac{P(\s(x)|y=1,\Dtgt)}{P(\s(x)|y=1,\Dtgt) + P(\s(x)|y=0,\Dtgt)}}_\text{we can use label shift} \otimes P(y=1|\Dtgt)
\\&= \underbrace{\frac{P(\s(x)|y=1)}{P(\s(x)|y=1) + P(\s(x)|y=0)}}_\text{invariant to $\Dtgt$}
\otimes P(y=1|\Dtgt)
\end{align*}
\begin{align*}
P(y=1|\s(x),\Dtgt) \otimes P(y=0|\Dtgt)
&= \underbrace{\frac{P(\s(x)|y=1)}{P(\s(x)|y=1) + P(\s(x)|y=0)}}_\text{invariant to $\Dtgt$}
\\&=P(y=1|\s(x),\Dsrc) \otimes P(y=0|\Dsrc)
\\P(y=1|\s(x),\Dtgt)
&=P(y=1|\s(x),\Dsrc) \otimes P(y=0|\Dsrc) \otimes P(y=1|\Dtgt)
\\&=P(y=1|\s(x),\Dsrc) \otimes (1-\pi_0) \otimes \pi
\end{align*}

Here, the propagation of errors is straightforward: the log odds error will be the same size under both distributions, although of course the errors in probability space may be larger or smaller.  Thus we can specify the best choice of prior-adjusted score:

\subsection{Prior-Adjusted, Cost-Weighted Threshold}

Combining these two, we find that the best choice of decision threshold is
\[
\pi \otimes (1-\pi_0) \otimes \s(x) \geq c
\]
We will thus often refer to the induced optimal classifier $\kp[\pi][c]$.
%!TEX root = ../main.tex
\section{Derivation of Set-based Metrics}
\label{apdx:set}

We start with the most popular family of evaluation methods, which are based on accuracy but include cost-sensitive generalizations.  These only require a binary classifier, which we define as a function $\kappa(x) : \mathcal{X} \to \{0, 1\}$.

They differ along two axes: the way they factor in the cost of errors, and the way they factor in the class balance of the dataset.

\begin{table}[h]
\centering
\begin{tabular}{lccc}
\toprule
& \textbf{Empirical} & \textbf{Balanced} & \textbf{Prior-Adjusted Maximum} \\
\midrule
\textbf{Accuracy} & $\metric{Accuracy}$ & $\metric{BA}$ & $\metric{PAMA}$ \\
\textbf{Weighted Accuracy} & $\metric{WA}$ & $\metric{BWA}$ & $\metric{PAMWA}$ \\
\textbf{Net Benefit} & $\metric{NB}$ & $\metric{BNB}$ & $\metric{PAMNB}$ \\
\bottomrule
\end{tabular}
\caption{Taxonomy of set-based evaluation metrics. Each row represents a different approach to handling error costs, and each column represents a different approach to handling class balance.  Note that when balanced, the second and third rows are equivalent.}
\label{tab:set-metrics}
\end{table}

\subsection{Accuracy}

\begin{minipage}{0.48\textwidth}
\begin{definition}[Accuracy]
  The accuracy of a thresholded binary classifier $\kappa(x, \tau)$ is given by:
  $$
  \metric{Accuracy}(\Dsrc, \s, \tau) = \sum_{x,y \in \Dsrc} V_\half(y, \kp)
  $$
\end{definition}
\end{minipage}
\begin{minipage}{0.48\textwidth}
\centering
\begin{tabular}{ccc}
\toprule
$V_\half(y, \widehat{y})$ & $\substack{y=1 \\ \text{(Syphilis)}}$ & $\substack{y=0 \\ \text{(No Syphilis)}}$ \\
\midrule
$\substack{\widehat{y}=1 \\ \text{(Treat)}}$ & 1 & 0 \\[1em]
$\substack{\widehat{y}=0 \\ \text{(Don't treat)}}$ & 0 & 1 \\
\bottomrule
\end{tabular}
\captionof{table}{Value function for Accuracy}
\label{tab:v-accuracy}
\end{minipage}

This is impractically neutral with regard to cost in that $V(y, \widehat{y} = 1 - y)$ is not a function of y, which corresponds to the contingency table in \Cref{tab:v-accuracy}.
It is practical but neither neutral nor flexible with regard to distribution shift in the sense that it implicitly assumes: $H(\Dtgt) = \delta(\Dtgt = \Dsrc)$.

The simplest way to make this more neutral is to evaluate on a balanced dataset, which we denote as $\Dhalf$.  Mechanically, we can draw from this dataset using importance sampling, if we assume $\Dtgt\rightarrow Y\rightarrow X$ and therefore $P(X|Y,\Dtgt) = P(X|Y)$.

\begin{minipage}{0.48\textwidth}
\begin{definition}[Balanced Accuracy]
  The balanced accuracy of a thresholded binary classifier $\kappa(x, \tau)$ is given by:
  \begin{align*}
    &\metric{BA}(\Dsrc, \s, \tau)
    \\&= \sum_{x,y \in \Dhalf} V_\half(y, \kp)
    \\&= \sum_{x,y \in \Dsrc} W(\pi_0 \to \half;y)V_\half(y, \kp[][c])
    \\&= \sum_{x,y \in \Dsrc} V(y, \kp)
  \end{align*}
\end{definition}
\end{minipage}
\begin{minipage}{0.48\textwidth}
\centering
\begin{tabular}{ccc}
\toprule
$V(y, \widehat{y})$ & $\substack{y=1 \\ \text{(Syphilis)}}$ & $\substack{y=0 \\ \text{(No Syphilis)}}$ \\
\midrule
$\substack{\widehat{y}=1 \\ \text{(Treat)}}$ & $\frac{1}{2\pi_0}$ & 0 \\[1em]
$\substack{\widehat{y}=0 \\ \text{(Don't treat)}}$ & 0 & $\frac{1}{2(1-\pi_0)}$ \\
\bottomrule
\end{tabular}
\captionof{table}{Value function for Balanced Accuracy}
\label{tab:v-balanced-accuracy}
\end{minipage}

If more flexibility is desired, at the expense of neutrality, it is necessary to evaluate at an arbitrary class balance.
Moreover, evaluating how well a classifier performs at one specific threshold is less useful than understanding how the best threshold performs at a specific class balance.

\begin{minipage}{0.48\textwidth}
\begin{definition}[Prior-Adjusted Maximum Accuracy]
  The prior-adjusted maximum accuracy given a scoring function $\s$ and a threshold $\tau$ with a class balance $\pi$ is given by:
  \begin{align*}
    &\metric{PAMA}(\Dtgt, \s, \tau)
    \\&= \sum_{x,y \in \Dtgt} V_\half(y, \kp)
    \\&= \sum_{x,y \in \Dsrc} W(\pi \to \half;y)V_\half(y, \kp[\pi])
    \\&= \sum_{x,y \in \Dsrc} V(y, \kp[\pi])
  \end{align*}
\end{definition}
\end{minipage}
\begin{minipage}{0.48\textwidth}
\centering
\begin{tabular}{ccc}
\toprule
$V(y, \widehat{y})$ & $\substack{y=1 \\ \text{(Syphilis)}}$ & $\substack{y=0 \\ \text{(No Syphilis)}}$ \\
\midrule
$\substack{\widehat{y}=1 \\ \text{(Treat)}}$ & $\frac{\pi}{\pi_0}$ & 0 \\[1em]
$\substack{\widehat{y}=0 \\ \text{(Don't treat)}}$ & 0 & $\frac{1-\pi}{1-\pi_0}$ \\
\bottomrule
\end{tabular}
\captionof{table}{Value function for Shifted Accuracy}
\label{tab:v-shifted-accuracy}
\end{minipage}

\subsection{Weighted Accuracy}
This problem is further complicated by the need to realistically confront asymmetric costs.
Consider the syphilis testing case: unnecessary treatment is 10 to 100 times less costly than a missed detection.
We will use 1/30 as a representative value for exposition, as the exact mechanics of syphilis testing are not central to this work.

First, we consider the balanced case, which is more mathematically tractable.

\begin{minipage}{0.48\textwidth}
\begin{definition}[Balanced Weighted Accuracy]
  The balanced weighted accuracy of a score function $\s$ with a threshold $\tau$ is given by:
  \begin{align*}
    &\metric{BWA}(\Dsrc, \s, \tau, c)
    \\&= \sum_{x,y \in \Dhalf} (1-c)^yc^{1-y}V_\half(y, \kp[][c])
    \\&= \sum_{x,y \in \Dsrc} W(\pi_0 \to 1-c;y)V_\half(y, \kp[][c])
    \\&= \sum_{x,y \in \Dsrc} V(y, \kp[][c])
  \end{align*}
\end{definition}
\end{minipage}
\begin{minipage}{0.48\textwidth}
\centering
\begin{tabular}{ccc}
\toprule
$V(y, \widehat{y})$ & $\substack{y=1 \\ \text{(Syphilis)}}$ & $\substack{y=0 \\ \text{(No Syphilis)}}$ \\
\midrule
$\substack{\widehat{y}=1 \\ \text{(Treat)}}$ & $\frac{1-c}{\pi_0}$ & 0 \\[1em]
$\substack{\widehat{y}=0 \\ \text{(Don't treat)}}$ & 0 & $\frac{c}{1-\pi_0}$ \\
\bottomrule
\end{tabular}
\captionof{table}{Value function for Balanced Weighted Accuracy}
\label{tab:v-balanced-weighted-accuracy}
\end{minipage}

The minimum possible value of this expression is clearly 0 if $V(y,\widehat{y})=0$ for all $y$.  The maximum is also clear:
$$
\sum_{x,y\in\Dsrc} W(\pi_0 \to \half;y)W(\half \to 1-c;y)\ind{}
= (1-c) + c = 1
$$
The obvious combination of the two weighting terms is not correct, however.
$$
\sum_{x,y\in\Dsrc} W(\pi_0 \to \half;y)W(\half \to 1-c;y)W(\half \to \pi;y)\ind{}
= \pi(1-c) + (1-\pi)c \neq 1
$$

The most intuitive approach involves rescaling the value of the true and false positives to be in the 1:30 ratio and then normalizing such that the maximum possible value remains 1 regardless of class balance.
This is known as Weighted Accuracy.
This procedure of normalizing the metric so that 0 is the worst possible value and 1 the best is generally known in the forecast evaluation literature as a skill score, but in the medical decisionmaking literature this particular metric is generally called Weighted Accuracy.

\begin{minipage}{0.48\textwidth}
\begin{definition}[Weighted Accuracy]
  The weighted accuracy of a thresholded binary classifier $\kappa(x, \tau)$ is given by:
  \begin{align*}
    &\metric{WA}(\Dsrc, \s, \tau)
    \\&= \frac{
      \sum_{x,y \in \Dsrc} (1-c)^yc^{1-y}V_\half(y, \kp)
    }{
      \sum_{x,y \in \Dsrc} (1-c)^yc^{1-y}V_\half(y,y)
    }
    \\&= \sum_{x,y \in \Dsrc} V(y, \kp)
  \end{align*}
\end{definition}
\end{minipage}
\begin{minipage}{0.48\textwidth}
\centering
\begin{tabular}{ccc}
\toprule
$V(y, \widehat{y})$ & $\substack{y=1 \\ \text{(Syphilis)}}$ & $\substack{y=0 \\ \text{(No Syphilis)}}$ \\
\midrule
$\substack{\widehat{y}=1 \\ \text{(Treat)}}$ & $\frac{1-c}{(1-c)\pi_0 + c(1-\pi_0)}$ & 0 \\[1em]
$\substack{\widehat{y}=0 \\ \text{(Don't treat)}}$ & 0 & $\frac{c}{(1-c)\pi_0 + c(1-\pi_0)}$ \\
\bottomrule
\end{tabular}
\captionof{table}{Value function for Weighted Accuracy}
\label{tab:v-weighted-accuracy}
\end{minipage}

\subsection{Net Benefit}
However, this makes comparisons across different class balances less meaningful, since as the class balance varies, the normalizing factor changes.
As a result, the effective value of a true positive changes.
One common approach from the Decision Curve Analysis literature is instead to normalize the true positive to 1 and then rescale the false positive to keep the right ratio.
The baseline in the DCA literature is to always predict the negative class, whereas the weighted accuracy literature uses a baseline of always predicting the wrong class.
Since this is equivalent up to constants, we will modify the parameterization of Net Benefit to make it more directly comparable.

\begin{minipage}{0.48\textwidth}
\begin{definition}[Net Benefit]
  The net benefit of a scoring function $\s$ with a threshold $\tau$ is given by:
  \begin{align*}
    &\metric{NB}(\Dsrc, \s, \tau, c)
    \\&= \sum_{x,y \in \Dsrc} \left(\frac{c}{1-c}\right)^{1-y}V_\half(y, \kp)
    \\&= \sum_{x,y \in \Dsrc} V(y, \kp)
  \end{align*}
\end{definition}
\end{minipage}
\begin{minipage}{0.48\textwidth}
\centering
\begin{tabular}{ccc}
\toprule
$V(y, \widehat{y})$ & $\substack{y=1 \\ \text{(Syphilis)}}$ & $\substack{y=0 \\ \text{(No Syphilis)}}$ \\
\midrule
$\substack{\widehat{y}=1 \\ \text{(Treat)}}$ & 1 & 0 \\[1em]
$\substack{\widehat{y}=0 \\ \text{(Don't treat)}}$ & 0 & $\frac{c}{1-c}$ \\
\bottomrule
\end{tabular}
\captionof{table}{Value function for Net Benefit}
\label{tab:v-net-benefit}
\end{minipage}

The disadvantage of this approach is that it's unintuitive that the net benefit of a perfect classifier is not reliably 1, and instead depends on the class balance.  The advantage is that when comparing at different class balances, the value of a true positive and a true negative remain fixed, so measurements are directly compared on the same scale.

\subsection{Prior-Adjusted Maximum Cost-Weighted Metrics}

We can combine the prior-adjusted maximum value approach with the cost-weighted metrics to get two new metrics that make sense to compare across different class balances.

\begin{minipage}{0.58\textwidth}
\begin{definition}[Prior-Adjusted Maximum Weighted Accuracy]
  \label{def:pamwa}
  The prior-adjusted maximum weighted accuracy given a scoring function $\s$ and a threshold $\tau$ with a class balance $\pi$ is given by:
  \begin{align*}
    &\metric{PAMWA}(\Dtgt, \s, \tau, c)
    \\&= \frac{
      \sum\limits_{\Dtgt} (1-c)^yc^{1-y}V_\half(y, \widehat{y})
    }{
      \sum\limits_{\Dtgt} (1-c)^yc^{1-y}V_\half(y, y)
    }
    \\&= \frac{
      \sum\limits_{\Dhalf} [(1-c)\pi]^y[c(1-\pi)]^{1-y}V_\half(y, \widehat{y})
    }{
      \sum\limits_{\Dhalf} [(1-c)\pi]^y[c(1-\pi)]^{1-y}V_\half(y, y)
    }
    \\&= \frac{
      \sum\limits_{\Dsrc} [(1-c)\pi(1-\pi_0)]^y[c(1-\pi)\pi_0]^{1-y}V_\half(y,\widehat{y})
    }{
      \sum\limits_{\Dsrc} [(1-c)\pi(1-\pi_0)]^y[c(1-\pi)\pi_0]^{1-y}V_\half(y,y)
    }
    \\&= \sum\limits_{\Dsrc} V(y, \kp[\pi])
  \end{align*}
\end{definition}
\end{minipage}
\begin{minipage}{0.38\textwidth}
\centering
\begin{tabular}{cccc}
  \toprule
  $V$ & $1$ & &  $0$ \\
  \midrule
  $1$ & \multicolumn{2}{l}{$\frac{1}{2\pi_0}\frac{(1-c)\pi}{(1-c)\pi + c(1-\pi)}$} & 0\\[1em]
  $0$ & 0 & \multicolumn{2}{r}{$\frac{1}{2(1-\pi_0)}\frac{c(1-\pi)}{(1-c)\pi + c(1-\pi)}$} \\
  \bottomrule
\end{tabular}
\captionof{table}{Value function for Prior-Adjusted Maximum Weighted Accuracy}
\label{tab:v-shifted-weighted-accuracy}
\end{minipage}

\begin{proposition}[PAMA Equivalence]
\label{prop:pamwa-pama}
\[
\metric{PAMWA}(\Dtgt, \s, \tau, c)
=\metric{PAMA}(\D_{1-c \otimes \pi}, \s, \tau)
\]
\begin{proof}
\begin{align*}
  &\metric{PAMWA}(\Dtgt, \s, \tau, c)
  \\&= \frac{
    \sum\limits_{\Dtgt} (1-c)^yc^{1-y}V_\half(y, \widehat{y})
  }{
    \sum\limits_{\Dtgt} (1-c)^yc^{1-y}V_\half(y, y)
  }
  \\&= \frac{
    \sum\limits_{\Dhalf} [(1-c)\pi]^y[c(1-\pi)]^{1-y}V_\half(y, \widehat{y})
  }{
    \sum\limits_{\Dhalf} [(1-c)\pi]^y[c(1-\pi)]^{1-y}V_\half(y, y)
  }
  \\&= \frac{
    \sum\limits_{\D_{1-c \otimes \pi}} V_\half(y, \widehat{y})
  }{
    \sum\limits_{\D_{1-c \otimes \pi}} V_\half(y, y)
  }
  \\&=\metric{PAMA}(\D_{1-c \otimes \pi}, \s, \tau)
\end{align*}
\end{proof}
\end{proposition}

\begin{minipage}{0.48\textwidth}
\begin{definition}[Prior-Adjusted Maximum Net Benefit]
  \label{def:pamnb}
  The prior-adjusted maximum net benefit of a scoring function $\s$ with a threshold $\tau$ with a class balance $\pi$ is given by:
  \begin{align*}
    &\metric{PAMNB}(\Dtgt, \s, \tau, c)
    \\&= \sum\limits_{\Dtgt} (\tfrac{\pi}{\pi_0})^{1-y}\,(\tfrac{c}{1-c}\tfrac{1-\pi}{1-\pi_0})^{1-y}\,V_\half(y, \widehat{y})
    \\&= \sum\limits_{\Dsrc} V(y, \kp[\pi])
  \end{align*}
\end{definition}
\end{minipage}
\begin{minipage}{0.48\textwidth}
\centering
\begin{tabular}{ccc}
\toprule
$V(y, \widehat{y})$ & $\substack{y=1 \\ \text{(Syphilis)}}$ & $\substack{y=0 \\ \text{(No Syphilis)}}$ \\
\midrule
$\substack{\widehat{y}=1 \\ \text{(Treat)}}$ & $\frac{\pi}{\pi_0}$ & 0 \\[1em]
$\substack{\widehat{y}=0 \\ \text{(Don't treat)}}$ & 0 & $\frac{c}{1-c}\frac{1-\pi}{1-\pi_0}$ \\
\bottomrule
\end{tabular}
\captionof{table}{Value function for Prior-Adjusted Maximum Net Benefit}
\label{tab:v-shifted-net-benefit}
\end{minipage}

We focus on the second because although the semantics of a single value are more confusing (since the perfect classifier is not normalized to 1), the values at different class balances are commensurable.

\subsection{Conclusion}

If we are willing to accept the causal diagram $\Dtgt\rightarrow Y\rightarrow X$, then we have tools available in different parts of the literature to broaden $V(y,\kappa(x,\tau))$ to capture asymmetric costs and move from $H(\Dtgt) = \delta(\D = \Dsrc)$ to $H(\D) = \delta(\D=\Dtgt)$ for any given $\pi$.
%!TEX root = ../main.tex
\section{Cost-sensitive Error Averaged Across Class Balances}
\label{apdx:cost}
A core obstacle to the use of proper scoring rules like the Brier Score to evaluate classifier performance under label shift in applied work is the ubiquity of asymmetric costs.
Traditional derivations of the equivalence between Mean Squared Error and Average Accuracy over a range of class balances have taken the CDF of the positive and negative classes as given.
These do not easily generalize to a world of asymmetric costs.
This appendix develops a rigorous mathematical framework for analyzing cost-weighted binary classifier performance averaged across a range of class balances.
The average of accuracy represents an integral over class balances of a sum over data of an $\ell^0$ ordering between the class balance and the model score on each data point.  The broad outline of the approach is as follows:
\begin{align*}
  &\underset{\ell^0(\pi)}{\smallint}
  \underbrace{
    \sum_{\Dtgt}
  }
  \ell^0(\s(x), \pi, y)
  d\pi
  \\[1em]&= \underset{\ell^0(\pi)}{\smallint}
  \overbrace{
    \underbrace{
      \sum_{\Dsrc}
      \W[\pi]
    }
  }^{\text{importance sampling}}
  \ell^0(\s(x), \pi, y)
  d\pi
  \\[1em]&= \underbrace{
    \underset{\ell^0(\pi)}{\smallint}
  }
  \overbrace{
    \underbrace{
      \sum_{\Dhalf}
    }
    \ell^1(\pi, y)
  }^{\text{importance weight is } \ell^1}
  \ell^0(\s(x), \pi, y)
  d\pi
  \\[1em]&= 
  \overbrace{
    \sum_{\Dhalf}
    \underbrace{
      \underset{\ell^0(\pi)}{\smallint}
    }
  }^\text{swap}
  \ell^1(\pi, y)
  \underbrace{
    \ell^0(\s(x), \pi, y)
  }
  d\pi
  \\[1em]&= 
  \sum_{\Dhalf}
  \overbrace{
    \underset{\ell^0(\pi) \cap \ell^0(\s(x), \pi, y)}{\smallint}
  }^\text{intersect}
  \underbrace{
    \ell^1(\pi, y)
  }
  d\pi
  \\[1em]&= 
  \sum_{\Dhalf}
  \overbrace{
    \ell^2(\s(x), y)
  }^\text{antidifferentiate}
  d\pi
\end{align*}
This flexible mathematical framework transforms a challenging integration problem into a tractable calculation using our core lemma. The key innovations (i.e., reframing importance weights and expressing classification correctness as set membership) provide deeper insight than the traditional integration-by-parts approach while generalizing beyond simple accuracy to various cost-sensitive losses. In the sections that follow, we demonstrate how this unified approach yields practical formulas for robust decision-making under class distribution uncertainty and asymmetric costs.

\subsection{Preliminaries}

See \Cref{apdx:label_shift_notation} for some of the notation used in this appendix.

\begin{definition}[Adjusted Score]
\[
  \s_{\pi'}(x) \triangleq \pi' \otimes \shalf(x)
\]
This represents the score optimally adjusted for the change in the prior probability.  See \Cref{apdx:label_shift} for more details.
\end{definition}

\begin{definition}[The set of $1-\pi$ for which $s$ gives the correct $\widehat{y}$]
\[
\chi(x,y) \triangleq \begin{cases}
  [1-y, \shalf(x)] & \text{if } y=1
  \\(\shalf(x), 1-y] & \text{if } y=0
\end{cases}
\]
\end{definition}
For proof that these are the right conditions, see \Cref{lem:adjust}.

\begin{lemma}[$\chi$ is the right set]
\label{lem:adjust}
\begin{align*}
1-\pi \in \chi(x,y)
\iff
y = \kp[\pi]
\end{align*}
\begin{proof}
We start by expressing the thresholding condition in terms of $1-\pi$.
\begin{align*}
& \pi \otimes (1-\pi_0) \otimes \s(x) \geq \tau
\\[1em]\implies& (1-\tau) \otimes (1-\pi_0) \otimes \s(x) \geq 1-\pi
\\[1em]\implies& \shalf(x) \geq 1-\pi
\\[1em]\implies& 1-\pi \in [0, \shalf(x)]
\end{align*}
Now we re-express the negation of this:
\begin{align*}
&\pi\otimes\shalf(x) \not\geq \tau
\\[1em]\implies& 1-\pi \not\in [0, \shalf(x)]
\\[1em]\implies& 1-\pi \in [\shalf(x), 1)
\end{align*}
We then compare to the label.
\begin{align*}
& \begin{cases}
  \pi\otimes\shalf(x) \geq \tau & \text{if } y=1
  \\\pi\otimes\shalf(x) < \tau & \text{if } y=0
\end{cases}
\\[1em]\implies& \begin{cases}
  1-\pi \in [1-y, \shalf(x)] & \text{if } y=1
  \\1-\pi \in (\shalf(x), 1-y] & \text{if } y=0
\end{cases}
\\[1em]\implies& 1-\pi \in \chi(x,y)
\end{align*}
\end{proof}
\end{lemma}

\begin{definition}[Clipping operator]
Also called projection or restriction in other contexts.
\[
\clip{a}{b}(x) \triangleq \max(a,\min(b,x))
\]
\end{definition}

\begin{lemma}[Intersecting $\chi$ with a closed interval]
This equality holds almost everywhere, because when the intersection is empty, the set on the right will be of measure zero.  Integrals of bounded functions over these two sets will give the same results.
\begin{align*}
&[a,b] \cap \chi(x,y)
\\[1em]&\overset{a.e.}{=}\begin{cases}
  [\clip{a}{b}(1-y),\; \clip{a}{b}(\shalf(x))] & \text{if } y=1
  \\ (\clip{a}{b}(\shalf(x)),\; \clip{a}{b}(1-y)] & \text{if } y=0
\end{cases}
\end{align*}
\end{lemma}

\subsection{Main Lemma}

\begin{lemma}[Average over Class Balance of a Sum over Data of a Cost-Weighted Correctness Condition]
\label{lem:clip}
\begin{align*}
&(b-a)\E_{\pi \sim \text{Uniform}(a,b)}
\E_{(x,y) \sim \Dtgt}
\C \yhateqy[\pi]
\\[1em]&=\E_{(x,y) \sim \Dhalf}
\int^{\clip{1-b}{1-a}(1-y)}_{p=\clip{1-b}{1-a}(\shalf(x))} 2(p - y) \Cp dp
\end{align*}
\end{lemma}

\begin{proof}
\newcommand{\CHI}{\ind{1-\pi \in \chi(x,y)}}
\begin{align*}
&\int_{\pi=a}^b \E_{(x,y) \sim \Dtgt} \C \yhateqy[\pi] d\pi
\\[1em]&=\int_{\pi=a}^b \E_{(x,y) \sim \Dtgt} \C \CHI d\pi
\\[1em]&=
\int_{\pi=a}^b \frac{1}{|\Dsrc|}\sum_{(x,y) \in \Dsrc} \W[\pi] \C \CHI d\pi
\\[1em]&=
\int_{\pi=a}^b \frac{1}{|\Dsrc|}\sum_{(x,y) \in \Dsrc} \W \;2\,|(1-\pi) - y|\, \C \CHI d\pi
\\[1em]&=
\frac{1}{|\Dsrc|} \sum_{(x,y) \in \Dsrc} \W
\int_{\pi=a}^b 2|(1-\pi) - y| \C \CHI d\pi
\end{align*}
We will now focus only on the inner integral.

\begin{align*}
&\int_{\pi=a}^b 2|(1-\pi) - y| \C \CHI d\pi
\\[1em]&=\int_{p=1-b}^{1-a} 2|p - y| \Cp \ind{p \in \chi(x,y)} dp
\\[1em]&=\underset{[1-b,1-a]\,\cap\,\chi(x,y)}{\smallint} 2|p - y| \Cp dp
\\&\qquad\text{by case analysis on y}
\\[1em]&=\int^{\clip{1-b}{1-a}(1-y)}_{p=\clip{1-b}{1-a}(\shalf(x))} 2(p - y) \Cp dp
\end{align*}
Combining the two parts, we get the result.
\begin{align*}
&(b-a)\E_{\pi \sim \text{Uniform}(a,b)}
\E_{(x,y) \sim \Dtgt}
\C \yhateqy[\pi]
\\[1em]&=
\int_{\pi=a}^b \frac{1}{|\Dsrc|}\sum_{(x,y) \in \Dtgt} \C \yhateqy[\pi] d\pi
\\[1em]&=
\frac{1}{|\Dsrc|} \sum_{(x,y) \in \Dsrc} \W
\int_{\pi=a}^b 2|(1-\pi) - y| \C \CHI d\pi
\\[1em]&=
\frac{1}{|\Dsrc|} \sum_{(x,y) \in \Dsrc} \W
\int^{\clip{1-b}{1-a}(1-y)}_{p=\clip{1-b}{1-a}(\shalf(x))} 2|p - y| \Cp dp
\\[1em]&=\E_{(x,y) \sim \Dhalf}
\int^{\clip{1-b}{1-a}(1-y)}_{p=\clip{1-b}{1-a}(\shalf(x))} 2(p - y) \Cp dp
\end{align*}
\end{proof}

\subsection{Example Use}

\begin{theorem}[Bounded Brier Score]
\begin{align*}
&(b-a)\E_{\pi \sim [a,b]} \metric{PAMA}(\Dtgt, \s, \tau)
\\[1em]&= \E_{(x,y) \sim \Dhalf}
\left[
\left(\clip{1-b}{1-a}(1-y)\;-y\right)^2-\left(\clip{1-b}{1-a}(\shalf(x))\;-y\right)^2
\right]
\end{align*}
\end{theorem}

\begin{proof}
Straightforward application of \Cref{lem:clip} with $\C = \mathbf{1}$.
\begin{align*}
&(b-a)\E_{\pi \sim [a,b]} \metric{PAMA}(\Dtgt, \s, \tau)
\\[1em]&= (b-a)\E_{\pi \sim [a,b]} \C \yhateqy[\pi]
\\[1em]&\qquad\text{ by \Cref{lem:clip}}
\\[1em]&= \E_{(x,y) \sim \Dhalf}
\int^{\clip{1-b}{1-a}(1-y)}_{p=\clip{1-b}{1-a}(\shalf(x))} 2(p-y)\C dp
\\[1em]&= \E_{(x,y) \sim \Dhalf}
\int^{\clip{1-b}{1-a}(1-y)}_{p=\clip{1-b}{1-a}(\shalf(x))} 2(p-y)dp
\\[1em]&\qquad\text{ by antidifferentiation}
\\[1em]&= \E_{(x,y) \sim \Dhalf}
\left[(p-y)^2\right]
_{p=\clip{1-b}{1-a}(\shalf(x))}
^{p=\clip{1-b}{1-a}(1-y)}
\\[1em]&= \E_{(x,y) \sim \Dhalf}
\left[
\left(\clip{1-b}{1-a}(1-y)\;-y\right)^2-\left(\clip{1-b}{1-a}(\shalf(x))\;-y\right)^2
\right]
\end{align*}
\end{proof}
%!TEX root = ../main.tex
\section{Log Scores with Asymmetric Costs}
\label{apdx:log}

\begin{theorem}[Bounded Log Score]
\label{thm:log}
\begin{align*}
&\left[\logit{b}-\logit{a}\right]
\E_{\logit{\pi} \sim [\logit{a}, \logit{b}]} \metric{PAMA}(\Dtgt, \s, \tau)
\\&= 2\E_{(x,y) \sim \Dhalf}
\log\left|\clip{1-b}{1-a}(1-y)\;-(1-y)\right|
-\log\left|\clip{1-b}{1-a}(\shalf(x))\;-(1-y)\right|
\end{align*}
\end{theorem}

\begin{proof}
Use \Cref{lem:clip} with $\C = \frac{1}{\pi(1-\pi)}$.
\begin{align*}
&\left[\logit{b}-\logit{a}\right]
\E_{\logit{\pi} \sim [\logit{a}, \logit{b}]} \metric{PAMA}(\Dtgt,\s, \tau)
\\[1em]&=
\left[b-a\right]
\E_{\pi \sim [a,b]}
\E_{(x,y) \sim \Dtgt}
\C \yhateqy[\pi]
\\[1em]&= \E_{(x,y) \sim \Dhalf}
\int^{\clip{1-b}{1-a}(1-y)}_{p=\clip{1-b}{1-a}(\shalf(x))} 2(p-y) \Cp dp
&\text{\Cref{lem:clip}}
\\[1em]&= \E_{(x,y) \sim \Dhalf}
\int^{\clip{1-b}{1-a}(1-y)}_{p=\clip{1-b}{1-a}(\shalf(x))} 2(p-y) \frac{dp}{(1-p)p}
\\[1em]&= \E_{(x,y) \sim \Dhalf}
\int^{\clip{1-b}{1-a}(1-y)}_{p=\clip{1-b}{1-a}(\shalf(x))} \frac{2dp}{(1-y-p)}
&\text{case analysis}
\\[1em]&= 2\E_{(x,y) \sim \Dhalf}
\ln\left|1-y-\clip{1-b}{1-a}(1-y)\right|
-\ln\left|1-y-\clip{1-b}{1-a}\shalf(x)\right|
\end{align*}
\end{proof}

\subsection{Weighted Accuracy}
Refer to \Cref{def:pamwa}

\begin{theorem}[PAMWA Log Score]
\label{thm:wa-log}
\begin{align*}
&
\left[\logit{b} - \logit{a}\right]
\E_{\logit{\pi} \sim [\logit{a}, \logit{b}]} \metric{PAMWA}(\Dtgt, \s, \tau, c)
\\[1em]&= 2\E_{(x,y) \sim \Dhalf}
\ln\left|
1 - y - \clip{c\otimes 1-b}{c\otimes 1-a}(1-y)
\right|
-\ln\left|
1 - y - \clip{c\otimes 1-b}{c\otimes 1-a}(\shalf(x))
\right|
\end{align*}
\end{theorem}

\begin{proof}
The high level idea is we replace a PAMWA term with a PAMA term by \Cref{prop:pamwa-pama}, and then do a change of variables, and apply \Cref{thm:log}.
\begin{align*}
&\left[\logit{b} - \logit{a}\right]
\E_{\logit{\pi} \sim [\logit{a}, \logit{b}]} \metric{PAMWA}(\Dtgt, \s, \tau, c)
\\[1em]&=
\left[\logit{b} - \logit{a}\right]
\E_{\logit{\pi} \sim [\logit{a}, \logit{b}]} \metric{PAMA}(\D_{1-c\otimes\pi},\s, \tau)
\tag{by \Cref{prop:pamwa-pama}}
\\&\qquad\text{now we do a change of variables }\pi' = (1-c)\otimes\pi
\\[1em]&=
\left[\logit{b} - \logit{a}\right]
\E_{\logit{\pi'} \sim [\logit{(1-c)\otimes a}, \logit{(1-c)\otimes b}]} \metric{PAMA}(\D_{\pi'}, \s, \tau)
\\[1em]&=
\left[\logit{1-c\otimes b} - \logit{1-c \otimes a}\right]
\E_{\logit{\pi'} \sim [\logit{(1-c)\otimes a}, \logit{(1-c)\otimes b}]} \metric{PAMA}(\D_{\pi'}, \s, \tau)
\tag{by \Cref{prop:logit-interval}}
\\&\qquad\text{now by \Cref{thm:log}}
\\[1em]&= 2\E_{(x,y) \sim \Dhalf}
\ln\left|
1 - y - \clip{1-(1-c)\otimes b}{1-(1-c)\otimes a}(1-y)
\right|
-\ln\left|
1 - y - \clip{1-(1-c)\otimes b}{1-(1-c)\otimes a}(\shalf(x))
\right|
\\[1em]&= 2\E_{(x,y) \sim \Dhalf}
\ln\left|
1 - y - \clip{c\otimes 1-b}{c\otimes 1-a}(1-y)
\right|
-\ln\left|
1 - y - \clip{c\otimes 1-b}{c\otimes 1-a}(\shalf(x))
\right|
\end{align*}
\end{proof}

\subsection{Net Benefit}
Refer to \Cref{def:pamnb}

\begin{theorem}[PAMNB Log Score]
\begin{align*}
&\left[\logit{b} - \logit{a}\right]
\E_{\logit{\pi} \sim [\logit{a}, \logit{b}]} \metric{PAMNB}(\Dtgt, \s, \tau, c)
\\[1em]&= \frac{2}{1-c}\E_{(x,y) \sim \D_{(1-c)}}
\ln\left|
1 - y - \clip{1-b}{1-a}(1-y)
\right|
-\ln\left|
1 - y - \clip{1-b}{1-a}(\s_{(1-c)}(x))
\right|
\end{align*}
\end{theorem}
\begin{proof}
\begin{align*}
&\left(\logit{b} - \logit{a}\right)
\E_{\logit{\pi} \sim [\logit{a}, \logit{b}]} \metric{PAMNB}(\Dtgt, \s, \tau, c)
\\&=
\int_{\logit{\pi} = \logit{a}}^{\logit{b}}
\metric{PAMNB}(\Dtgt, \s, \tau, c)
\;d\logit{\pi}
\\&=
\int_{\pi = a}^{b}
\metric{PAMNB}(\Dtgt, \s, \tau, c)
\;\frac{d\pi}{\pi(1-\pi)}
\\&=
(b-a)
\E_{\pi \sim [a,b]}
\frac{\metric{PAMNB}(\Dtgt, \s, \tau, c)}{\pi(1-\pi)}
\\&=
(b-a)
\E_{\pi \sim [a,b]}
\E_{x,y\sim \Dtgt}
\frac{
  \left(\frac{c}{1-c}\right)^{1-y}
  \yhateqy[\pi][c]
}{\pi(1-\pi)}
\\&=
(b-a)
\E_{\pi \sim [a,b]}
\E_{x,y\sim \Dtgt}
\frac{1}{\pi(1-\pi)}
\frac{
  c^{1-y}
  (1-c)^y
}{(1-c)}
\yhateqy[\pi\otimes(1-c)]
\\&\qquad\text{now we use \Cref{lem:clip}}
\\[1em]&=
\E_{(x,y) \sim \Dhalf}
\int^{\clip{1-b}{1-a}(1-y)}_{p=\clip{1-b}{1-a}(1-c\otimes\shalf(x))} 2(p - y)
\frac{1}{(1-p)p}\frac{c^{1-y}(1-c)^y}{1-c} dp
\\[1em]&= \frac{2}{1-c}\E_{(x,y) \sim \Dhalf} c^{1-y}(1-c)^y
\int^{\clip{1-b}{1-a}(1-y)}_{p=\clip{1-b}{1-a}((1-c)\otimes \shalf(x))} \frac{p-y}{(1-p)p} dp
\\&\qquad\text{and we use importance sampling}
\\[1em]&= \frac{2}{1-c}\E_{(x,y) \sim \D_{1-c}}
\int^{\clip{1-b}{1-a}(1-y)}_{p=\clip{1-b}{1-a}(\s_{1-c}(x))} \frac{p-y}{(1-p)p} dp
\\&\qquad\text{and by case analysis on y}
\\[1em]&= \frac{2}{1-c}\E_{(x,y) \sim \D_{(1-c)}}
\int^{\clip{1-b}{1-a}(1-y)}_{p=\clip{1-b}{1-a}(\s_{(1-c)}(x))} \frac{1}{1-y-p} dp
\\[1em]&=
\frac{2}{1-c}\E_{(x,y) \sim \D_{(1-c)}}
\ln\left|
1 - y - \clip{1-b}{1-a}(1-y)
\right|
-\ln\left|
1 - y - \clip{1-b}{1-a}(\s_{(1-c)}(x))
\right|
\end{align*}
\end{proof}
%!TEX root = ../main.tex
\section{13 ways of looking at an AUC-ROC}
\label{apdx:13}

For largely contingent historical reasons, medical informaticists have long reported results either in terms of specificity and sensitivity (if they want a fixed threshold), or in terms of the Area Under the Receiver Operating Characteristic Curve (AUC-ROC).  What exactly does this mean?  There are 12 standard interpretations of the AUC-ROC, and none of them fit very well.  We add a thirteenth which is more relevant but still inadequate for our purposes.

\begin{enumerate}
  \item 0.5 when the classifier is random, and 1.0 when the classifier is perfect. \citet{balancedaccuracy23} reports this as the most common interpretation, a bit tongue in cheek.  It is, unfortunately, also the current authors' experience that this is the most commonly given interpretation in practice.

  \item The 2 alternative forced choice accuracy rate \cite{swetsbirdsall56}.

  This only makes sense in the original psychometric setting where an experimenter in fact guarantees that there is one positive and one negative case \cite{balancedaccuracy23}.

  \item A rescaled version of the Mann-Whitney $U$ statistic \cite{mannwhitney75,hanley82}.  This is actually the same as the statement above, but it sounds more impressive.  Note that AUC-ROC is never reported as a p-value based on this statistic, which suggests that the interpretation is not practically very useful.

  \item A rescaled version of the Kendall's $\tau$ correlation coefficient \cite{hernandez13rate}.

  It is technically true that the AUC-ROC is a pairwise permutation distance between the ideal ranking and the actual ranking.  But there are only 2 ranks!  This makes the exercise meaningless.

  \item An average of precision (though not "Average Precision" which refers to something else)

  There is an occasional attempt to rescue the paradigm by arguing that AUC-ROC shows an average of $TP = K\times \text{Precision}@K$ over a range of $K$.  The trouble is twofold:
    \begin{itemize}
      \item This gives a uniform average over all possible values of $K$.  In a quantity constrained setting where we're forced to pick out only $K$ items to give positive labels (imagine only 10 doses of penicillin in the freezer) we generally have very small K.
      \item Few binary classification problems are actually intended for explicit quantity constraints, so this doesn't fit well.
    \end{itemize}

  \item An average of power over a range of sizes (in the Neyman-Pearson sense)

  This one is very popular with practitioners and virtually absent from the literature, aside from \cite{mcclish89}, which was later criticized by \cite{mcclish12,balancedaccuracy23}.  The trouble is that in the Neyman-Pearson paradigm we're meant to pick a power, and then find out what the size of the test is.  Even if we reverse this interpretation as an average size over a range of powers, the power in the NP paradigm isn't an empirical quantity we might observe a distribution over.  We're meant to pick one.

  \item It is the area under a curve if FPR is plotted against TPR. \cite{hanley82}
  
  \item Average accuracy on the positive class across a uniform distribution of accuracy on the negative class, or vice versa \cite{metz86,metz89,mcclish02}.  This is actually the same as the statement above, but more useful-sounding, and slightly less mysterious since it doesn't use the words "False Positive Rate" or "True Positive Rate".

  \item Given two thresholds $a < b$, the average accuracy on the positive class across a uniform distribution of accuracy on the negative class between those two thresholds, plus the average accuracy on the negative class across a uniform distribution of accuracy on the positive class between those two thresholds, weighted by the class balance between those two thresholds \cite{balancedaccuracy23}.  This is a bit more useful, but it's not clear why we're using the exact class balance between the thresholds.

  \item An average of accuracy as we set the threshold at each data point, leaving operating conditions the same \cite{hernandez11coherent}.  This interpretation is the only one from this set of authors that directly addresses the problem that data is empirical and discrete.

  \item An average of cost-weighted error over a range of cost ratios.
      \citet{hand09} shows that if a score is calibrated, then the AUC-ROC is an average of the cost-weighted error over a range of prevalences. The trouble is that:
      \begin{itemize}
        \item Calibration is a really important property of a score!  Without it, we can only make top-K decisions.  Even then we have to assume there are more than $K$ real positives.
        \item The range of costs is a function of model scores on the training data.  But the risk that a given individual has syphilis is not a good estimator of how much less harmful unnecessary penicillin is than untreated tertiary syphilis.
        \item Because the costs are a function of the model, if two people train models on the same data, they will get different costs, and the average accuracies will be incommensurable.
      \end{itemize}

  \item An average of skew-weighted cost, for skew $z=c\otimes(1-\pi)$.
    \cite{hernandez13rate} proposes this interpretation, although again the costs are set arbitrarily and distinctly by each model, and class-conditional distributions are assumed to be continuous and fully known.  Furthermore, it is not clear what it means to integrate over values of $c\otimes(1-\pi)$.  One possible interpretation is that we set $c=(1-\pi)$ and then calculate a line integral along a particular curve in cost / prevalence space, by way of estimating the area integral over the entire space.

  \item An average of accuracy under label shift, where the distribution of positive class prevalences is derived by sampling from the model scores on the training data.  This is syntactically similar, but specifically derived in the case of the sampling problems that arise from label shift (see \Cref{thm:auc-roc}).  While this is the most relevant definition for our problem, it shares the issues of the distribution being unrelated to the class balance of the deployment and incommensurable across models.

\end{enumerate}
%!TEX root = ../main.tex
\section{AUC-ROC as Average Accuracy under Label Shift (symmetry, no precision)}
See \Cref{apdx:13} for an enumeration of alternative interpretations of the AUC-ROC, an idea inspired by \citet{turakhia2017}.  Here we focus on the interpretation as an average of shifted accuracy, when the class balance distribution is derived by taking one minus the score from samples drawn from the training data, assuming that the classifier is calibrated, so that $\mathbb{P}(y=1 \mid \s(x)=\tau)=\tau$ for all $\tau \in [0,1]$.

As discussed in \Cref{sec:aucroc}, this interpretation shows a clear meaning for AUC-ROC in our context, but also shows its shortcomings.

Our key innovation is to focus on using importance sampling to directly address label shift.  We begin with a definition, which we then reexpress in terms of balanced classes.  Then we use importance sampling to reexpress components of the AUC-ROC in terms of true positives or true negatives.  Finally, thanks to the intrinsic class symmetry of the AUC-ROC, we combine the results.

\begin{definition}[AUC-ROC as Average over Data]
\begin{align*}
\\\metric{AUC-ROC}(\Dsrc, \s)
&\triangleq
\sum_{(x,y) \in \Dsrc} \frac{1}{|\Dsrc|} \frac{1-y}{1-\pi_0}
\sum_{(x',y') \in \Dsrc} \frac{1}{|\Dsrc|} \frac{y'}{\pi_0}
\Big[\ind{\s(x') > \s(x)} + \half\ind{\s(x') = \s(x)}\Big]
\end{align*}
\end{definition}

We now show that this is equivalent to drawing from a class-balanced distribution, a known property of the AUC-ROC, but one which will simplify our analysis.

\begin{lemma}[Importance Weighting]
\begin{align*}
&
\frac{1}{|\Dsrc|^2}
\sum_{(x,y) \in \Dsrc}
\sum_{(x',y') \in \Dsrc}
\frac{1-y}{1-\pi_0}
\frac{y'}{\pi_0}
\Big[\ind{\s(x') > \s(x)}
+ \half\ind{\s(x') = \s(x)}\Big]
\\&=
\E_{(x,y) \in \Dhalf}
\E_{(x',y') \in \Dhalf}
(1-y)y'
\Big[\ind{\s(x') > \s(x)} + \half\ind{\s(x') = \s(x)}\Big]
\end{align*}
\end{lemma}

In the continuous case, the $\half\ind{\s(x')=\s(x)}$ term disappears.  In the discrete case, the AUC-ROC is neither exactly equal to the average of true positives nor the average of true negatives, but instead to the average of the two.

\begin{lemma}[Average over Label Shifts of True Positive Rate]
\begin{align*}
&
\E_{(x,y) \in \Dhalf}
\E_{(x',y') \in \Dhalf}
(1-y)y'
\Big[\ind{\s(x') > \s(x)} + \ind{\s(x') = \s(x)}\Big]
\\&
\E_{t \in \s[\Dhalf]}
\E_{(x',y') \in \Dhalf}
\left[\E_{(x,y) \in \Dhalf : \s(x) = t}
1-y
\right]
y'
\Big[\ind{\s(x') \geq t}\Big]
\\&=
\E_{t \in \s[\Dhalf]}
\left[\E_{(x,y) \in \Dhalf}\ind{\s(x) = t}\right]
\E_{(x',y') \in \Dhalf}
(1-t)\cdot y'
\Big[\ind{\s(x') \geq t}\Big]
\\&=
\E_{t \in \s[\Dhalf]}
\E_{(x',y') \in \D_{1-t}}
y'\cdot \ind{\s(x') \geq t}
\end{align*}
\end{lemma}

By reversing the order of summation, we can also express a related quantity as an average of true negatives, instead.

\begin{lemma}[Average over Label Shifts of True Negative Rate]
\begin{align*}
&
\E_{(x,y) \in \Dhalf}
\E_{(x',y') \in \Dhalf}
(1-y)y'
\Big[\ind{\s(x') > \s(x)}\Big]
\\&=
\E_{(x',y') \in \Dhalf} 
\E_{(x,y) \in \Dhalf}
y' (1-y)
\Big[\ind{\s(x') > \s(x)}\Big]
\\&=
\E_{t \in \s[\Dhalf]}
\E_{(x,y) \in \Dhalf}
\left[\E_{(x',y') \in \Dhalf : \s(x') = t}
y'
\right]
(1-y)
\Big[\ind{t > \s(x)}\Big]
\\&=
\E_{t \in \s[\Dhalf]}
\left[\E_{(x',y') \in \Dhalf}\ind{\s(x') = t}\right]
\E_{(x',y') \in \Dhalf}
t (1-y)
\Big[\ind{t > \s(x)}\Big]
\\&=
\E_{t \in \s[\Dhalf]}
\E_{(x,y) \in \D_{1-t}}
(1-y)\cdot \ind{\s(x) < t}
\end{align*}
\end{lemma}

Finally, we combine these three results to show that the AUC-ROC is an average of accuracy under label shift.

\begin{theorem}[AUC-ROC As Average Accuracy]
\label{thm:auc-roc}
\begin{align*}
2\metric{AUC-ROC}(\s)
&=\E_{t \sim \s[\Dhalf]} \metric{PAMA}(\D_{1-t}, \s, \half)
\end{align*}
\begin{proof}
\begin{align*}
&2\metric{AUC-ROC}(\Dsrc, \s)
\\&=
\E_{(x,y) \in \Dhalf}
\E_{(x',y') \in \Dhalf}
(1-y)y'
\Big[2\cdot\ind{\s(x') > \s(x)} + \ind{\s(x') = \s(x)}\Big]
\\&=
\E_{(x,y) \in \Dhalf}
\E_{(x',y') \in \Dhalf}
(1-y)y'
\Big[\ind{\s(x') > \s(x)} + \ind{\s(x') = \s(x)}\Big]
\\&\qquad\qquad+
\E_{(x,y) \in \Dhalf}
\E_{(x',y') \in \Dhalf}
(1-y)y'
\Big[\ind{\s(x') > \s(x)}\Big]
\\&=
\E_{t \sim \s[\Dhalf]}\left[
\E_{(x',y') \in \D_{1-t}}
y'\cdot \ind{\s(x') \ge t}
+
\E_{(x,y) \in \D_{1-t}}
(1-y)\cdot \ind{\s(x) < t}
\right]
\\&=
\E_{t \sim \s[\Dhalf]}\left[
\metric{Accuracy}(\D_{1-t}, \s(x), \tau=t)
\right]
\\&=
\E_{t \sim \s[\Dhalf]}\left[
\metric{Accuracy}(\D_{1-t}, 1-t\otimes\s(x), \tau=\half)
\right]
\\&=
\E_{t \sim \s[\Dhalf]}\left[
\metric{PAMA}(\D_{1-t}, \s(x), \tau=\half)
\right]
\end{align*}
\end{proof}
\end{theorem}
\section{EICU}
\label{apdx:eicu}
We show how to use our approach to evaluate subgroup performance on the EICU dataset.  We consider the in-hospital mortality predictions based on APACHE IV scores in the freely available subset of the eICU dataset \cite{eicu00,eicu18,eicu19,eicu21}.
A quick check shows that the AUC-ROC and Accuracy are both higher for the African American patients than for the Caucasian patients.
However, the clipped cross entropy is lower.
We use further decompositions to understand why this is happening, and why the metrics fail to agree.

\subsection{Mechanism Shift vs. Label Shift}
To separate which parts of the difference in log score are due to different class balance vs different mechanism of prediction, we can rely on the label shift assumption (within subgroups) and use the do calculus of \citet{pearl00}
\begin{figure}[ht]
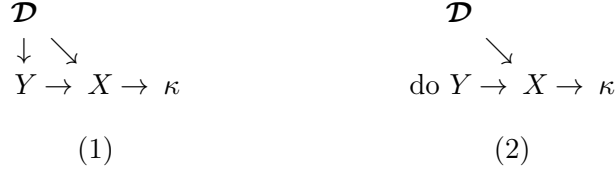

\centering
\begin{tabular}{ccc}
\begin{minipage}{0.31\textwidth}
\centering
$$
\begin{array}{r@{\,}r@{\,}r}
  \D & &\\
  \downarrow\,&\searrow\quad\,&\\
   Y & \rightarrow X & \rightarrow \kappa
\end{array}
$$
(1)
\end{minipage} &
\begin{minipage}{0.31\textwidth}
\centering
$$
\begin{array}{r@{\,}r@{\,}r}
  \D & &\\
  &\searrow\quad\,&\\
  \text{do } Y & \rightarrow X & \rightarrow \kappa
\end{array}
$$
(2)
\end{minipage}
\end{tabular}
\caption{
  Causal diagrams: (1) shows differences in performance are based both on label shift and differences in mechanism between subgroups (2) if we can intervene on Y to hold it constant, then we can measure differences in performance based purely on differing performance of the model between subgroups.  As covered in the main body of the text, under label shift, we can simulate changes in Y using importance weighting and prior-adjusted maximal accuracy.
}
\label{fig:causal-subgroup}
\end{figure}

For brevity, all our expectations over $\pi$ in this section will assume the distribution \[
\logit{\pi} \sim \text{Uniform}(\logit{\pi_\text{blue}}, \logit{\pi_\text{orange}})
\]
The causal impact after intervention on Y is given by the clipped cross entropy:
\[
\Delta_{\D\to X} =
\E_\pi \text{PAMNB}(\D_{\text{orange},\pi}, \s, c) - \text{PAMNB}(\D_{\text{blue},\pi}, \s, c)
\]
We can now additively decompose the difference in Accuracy:
\[
\Delta_{\D\to \kappa} = \Delta_{\D\to X} + \Delta_{\D\to Y}
\]
Which allows us to calculate the label shift effect:
\begin{align*}
\Delta_{\D\to Y} =
\E_\pi \bigg[
&\Big[
  \text{PAMNB}(\D_{\text{orange},\pi}, \s, c) -
\text{PAMNB}(\D_{\text{orange}}, \s, c)
\Big]
\\+
&\Big[
\text{PAMNB}(\D_{\text{blue}}, \s, c)
- \text{PAMNB}(\D_{\text{blue},\pi}, \s, c)
\Big]
\bigg]
\end{align*}

\begin{figure}[ht]
\centering
\includegraphics[width=0.4\textwidth]{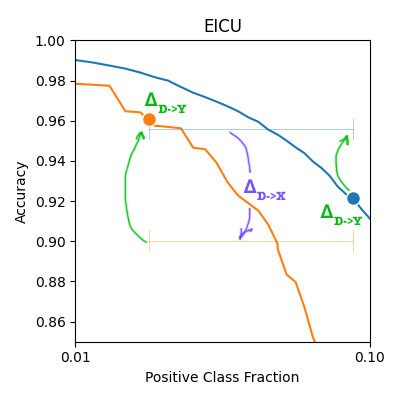}
\includegraphics[width=0.4\textwidth]{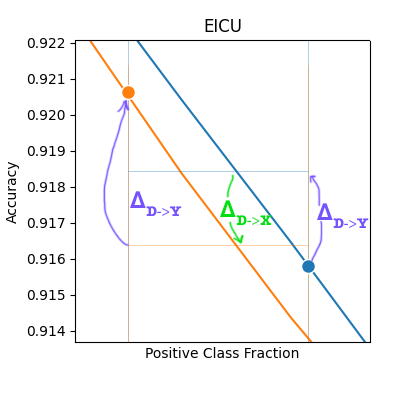}
\caption{
  The label shift effect is quite dramatic on the public subsample of EICU.
  In fact a great deal of this results from small sample size.
  But on the full sample, at a much finer scale, the same effect is present.
  $\Delta_{\D\to \kappa} > 0$, but the difference is all from $\Delta_{\D\to Y}$, and in fact $\Delta_{\D\to X} < 0$.
}
\end{figure}

\subsection{Calibration vs. Sharpness}

We can also decompose the difference in accuracy into calibration loss and sharpness.

\[
  \Delta_{TOTAL} = \Delta_{S} + \Delta_{C}
\]

\begin{figure}[ht]
\centering
\includegraphics[width=0.4\textwidth]{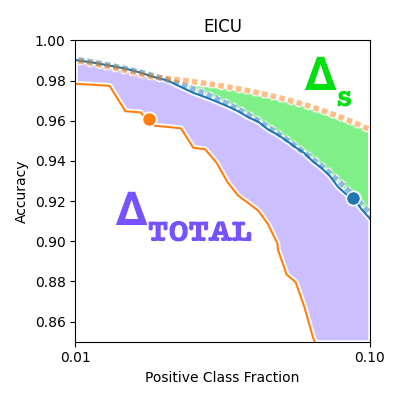}
\includegraphics[width=0.4\textwidth]{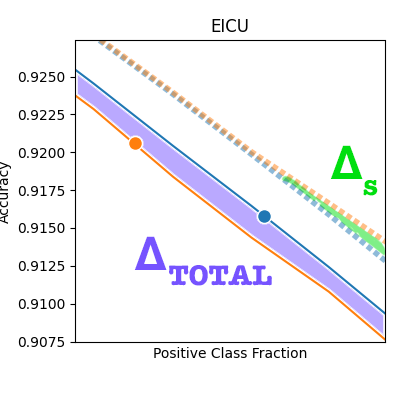}
\caption{
  The difference in sharpness is fairly dramatic on the public subsample of EICU,
  but swamped by the difference in calibration (which is far worse for African Americans).
  Something far less dramatic happens on the full sample, where calibration is quite good, and most of the loss comes from sharpness.  The sharpness of the model is still better for African Americans, but the miscalibration is worse, and that effect dominates in the overall.
}
\end{figure}

It's easy to calculate the overall gap.
\[
\Delta_{TOTAL} = \E_\pi \text{PAMNB}(\D_{\text{orange},\pi}, \s, c) - \text{PAMNB}(\D_{\text{blue},\pi}, \s, c)
\]
The sharpness gap is also straightforward, since we can simply recalibrate the model on the subgroup of the evaluation set (which we'll call $s^*$), and measure the loss of the recalibrated model.
\[
\Delta_{S} = \E_\pi \text{PAMNB}(\D_{\text{orange},\pi}, \s^*_\text{orange}, c) - \text{PAMNB}(\D_{\text{blue},\pi}, \s^*_\text{blue}, c)
\]
This remaining gap is the calibration gap.
\begin{align*}
\Delta_{C} =
\E_\pi \bigg[
&\Big[
  \text{PAMNB}(\D_{\text{orange},\pi}, \s, c) -
\text{PAMNB}(\D_{\text{orange},\pi}, \s^*_\text{orange}, c)
\Big]
\\+
&\Big[
\text{PAMNB}(\D_{\text{blue},\pi}, \s^*_\text{blue}, c)
- \text{PAMNB}(\D_{\text{blue},\pi}, \s, c)
\Big]
\bigg]
\end{align*}

Of course the difference in sharpness measured by clipped cross entropy is not exactly the same as the difference in sharpness measured by the AUC-ROC, because the weighting across thresholds is different.  Nevertheless, the two are correlated.

\begin{figure}[ht]
\centering
\includegraphics[width=0.4\textwidth]{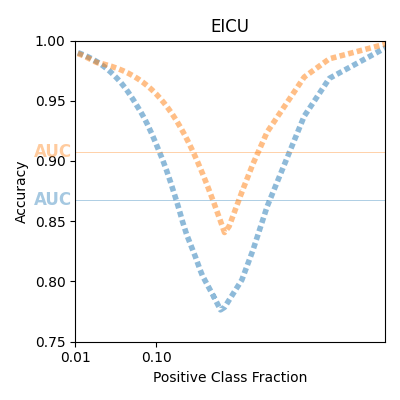}
\includegraphics[width=0.4\textwidth]{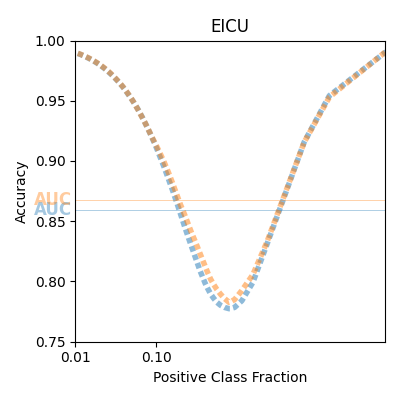}
\caption{
  The difference in sharpness for clipped cross entropy is not exactly the same as the difference in sharpness measured by the AUC-ROC curve because, as mentioned earlier, the AUC-ROC weights different thresholds differently.  However, there is some correlation between the two.
  Sharpness across the full range of prevalences is a fair bit higher for African Americans in the public subsample, and somewhat higher in the full sample as well.
}
\end{figure}

\subsection{Confidence Intervals}

As is often the case, it's easy to take this too seriously.  Since the prevalence of mortality is low, and the fraction of black patients is already low as well, it turns out there's a tremendous amount of sampling error in the public subset.  When we draw the confidence intervals, we see that we can't really conclude anything.

\begin{figure}[ht]
    \centering
    \includegraphics[width=0.45\textwidth]{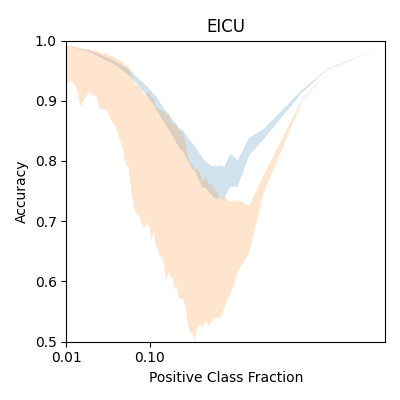}
    \includegraphics[width=0.45\textwidth]{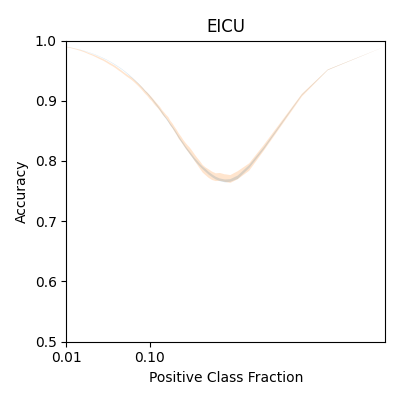}
    \caption{
      There's a significant gap between the performance on the public dataset for black and white patients, but it's well within the 95\% confidence interval.
      By comparison, when we use the full dataset, we see that the quality of the models is actually quite similar, and any difference in clinical utility is largely due to the difference in prevalence.
    }
    \label{fig:zoom}
\end{figure}

\subsection{Conclusion}

While AUC-ROC and Accuracy each handle some aspects of performance, we show an example with well known public data where both give the same misleading intuition about subgroup performance differences.  What makes the clipped cross entropy approach useful in this setting is that its linearity makes it easy to decompose different effects, and its integral across scenarios allows us to catch effects that can't be seen at any one class balance.

It is also important that in each of these decomposition plots the two effects we were trying to decompose were measured in the same units, so we don't have to ask ourselves which one is more important.  This generalizes to Net Benefit and Weighted Accuracy decompositions as well.
\end{document}